\pdfoutput=1

\newif\ifextendedversion

\extendedversiontrue

\documentclass{article}
\usepackage{ijcai21}

\usepackage{times}
\usepackage{soul}
\usepackage{url}
\usepackage[hidelinks]{hyperref}
\usepackage[utf8]{inputenc}
\usepackage[small]{caption}
\usepackage{graphicx}
\usepackage{amsmath}
\usepackage{amsthm}
\usepackage{booktabs}
\usepackage{algorithm}
\usepackage{algorithmic}
\renewcommand{\vec}{\mathbf}
\usepackage{amsfonts}
\usepackage{ifdraft}
\usepackage{tikz}
\ifdraft{
  \usepackage{color}
}{}

\usepackage[inline]{enumitem}
\newlist{romanenumerate*}{enumerate*}{1}
\setlist[romanenumerate*]{label=(\textit{\roman*})}
\newlist{romanenumerate}{enumerate}{1}
\setlist[romanenumerate]{label=(\textit{\roman*})}

\usepackage{thm-restate}

\newtheorem{example}{Example}

\newtheorem{lemma}{Lemma}
\newtheorem{proposition}{Proposition}

\newtheorem{definition}{Definition}

\title{Efficient PAC Reinforcement Learning in Regular Decision Processes}
\author{
  Alessandro Ronca \And Giuseppe De Giacomo
  \affiliations
  DIAG, Sapienza University of Rome, Italy
  \emails
  \{ronca, degiacomo\}@diag.uniroma1.it
}

\usepackage{xcolor}

\usepackage{mymacros}

\begin{document}

\maketitle 

\begin{abstract}
  Recently \emph{regular decision processes} have been proposed as a
  well-behaved form of \emph{non-Markov decision process}.
  \emph{Regular decision processes} are characterised by a transition
  function and a reward function that depend on the whole history, though
  \emph{regularly} (as in regular languages). In practice both the transition
  and the reward functions can be seen as finite transducers.
  We study \emph{reinforcement learning} in regular decision
  processes. Our main contribution is to show that a near-optimal policy can be
  PAC-learned in polynomial time in a set of parameters that describe the
  underlying decision process. We argue that the identified set of
  parameters is minimal and it reasonably captures the difficulty of a regular
  decision process.
\end{abstract}

\section{Introduction}

The standard RL setting  \cite{suttonbarto} has long been characterised by the
Markov assumption, which is arguably limiting. 
\emph{Regular decision processes} (RDPs) have been proposed as a
formalism with the promise to dispense with the Markov assumption while
retaining good computational properties \cite{brafman2019rdp}.
Complete knowledge of an RDP allows one to compute an equivalent Markov decision
process (MDP), in a precise formal sense, and then to apply standard solution
techniques for MDPs.  This correspondence is not immediately applicable in the
standard reinforcement learning (RL) setting, where an agent has no prior
knowledge of the decision process.
Instead, one needs to learn simultaneously a model
of the dependencies on the past, and
how observations and rewards evolve \cite{brafman2019rdp}.
The possibility for an RL agent to learn such a model in the form of an
automaton has been exploited in \cite{abadi2020learning}. The work
provides empirical evidence that algorithms based on automata
learning can find near-optimal policies for certain RDPs.
However, the proposed technique requires exponential time to converge
unless only a small number of traces have high probability to be observed:
it relies on statistics for single traces that may have exponentially-low
probability, and hence require exponential time to be observed a
number of times that is sufficient to compute accurate statistics.


In this paper, for the first time to the best of our knowledge, we introduce a
technique for RL in RDPs which has formal guarantees analogous to those
studied in literature for RL in MDPs 
\cite{fiechter1994efficient,brafman2002rmax,kearns2002near,auer2002finite,kakade2003thesis,strehl2005theoretical,auer2006ucrl,strehl2009reinforcement,audibert2009exploration,jaksch2010near,szita2010mormax,lattimore2014near,dann2015sample,dann2017unifying,azar2017minimax,bai2020near,wang2020long}.

We show that RL in RDPs is feasible in \emph{polynomial time} with
respect to the \emph{probably approximately correct} (PAC) learning criterion
\cite{valiant1984theory,kearns1994introduction}. 
We present an algorithm that computes a near-optimal policy with high
confidence, in a number of steps that is polynomial in the required accuracy and
confidence, and in a set of parameters that describe the underlying RDP.
The algorithm first learns a \emph{probabilistic deterministic finite automaton}
(PDFA) that represents the underlying RDP, then it maps it to an MDP, which is
solved to compute an intermediate stationary policy, which is then turned into
the final policy by composing it with the transition function of the learned
PDFA.
This process is repeated many times, yielding a better policy as more data
becomes available, and converging in polynomial time.

Our results show that the Markov assumption can be
largely removed while maintaining polynomial-time guarantees. This is a
fundamental positive result, that should boost the confidence in finding
effective RL techniques that dispense with the Markov assumption.
Furthermore, our results should draw attention on RDPs, which have the potential
of becoming the standard formalism for RL in the non-Markov setting, similarly
to MDPs in the Markov setting.

Our work shows a strong connection between RDPs and PDFA, which allows one
to take advantage of the many fundamental results and learning algorithms
available for PDFA
\cite{kearns1994learnability,ron1998learnability,clark2004pac,palmer2007pac,balle2013pdfa,balle2014adaptively}.
On one hand, the negative results from \cite{kearns1994learnability}
about the impossibility of polynomial-time learning when
\emph{state distinguishability} is low transfer to RDPs.
On the other hand, when state distinguishability is sufficiently high, PDFA
techniques allow for polynomial-time learning (cf.\
\cite{ron1998learnability,balle2013pdfa}), and we show that such techniques can
be leveraged to learn RDPs.
In particular, we claim that PDFA techniques are the only known techniques so
far to effectively learn states when they determine probability distributions,
as it happens in RDPs.

\ifextendedversion
  Complete proofs of all our technical results are deferred to the appendix.
\else
  Proofs of all results are given in an extended version of this paper
  \cite{extendedversion}
\fi

\section{Preliminaries}

\paragraph{Transducers.}
We follow \cite{moore1956gedanken}.
A \emph{transducer} is a tuple 
$\langle Q, q_0, \Sigma, \tau, \Gamma, \theta \rangle$ where:
$Q$ is a finite set of states; 
$q_0 \in Q$ is the initial state; 
$\Sigma$ is the finite input alphabet;
$\tau: Q \times \Sigma \to Q$ is the deterministic transition function; 
$\Gamma$ is the finite output alphabet; 
$\theta: Q \to \Gamma$ is the output function.
We extend the use of $\tau$ 
to strings of length greater than one as
$\tau(q,\sigma_1\sigma_2 \dots \sigma_n) = 
\tau(\tau(q,\sigma_1), \sigma_2 \dots\sigma_n)$,
and to the empty string as 
$\tau(q,\varepsilon) = q$.
We also extend the use of $\theta$ to arbitrary strings as
$\theta(q,\sigma_1\dots\sigma_n)  =
\theta(q)\, \theta(\tau(q,\sigma_1),\sigma_2 \dots \sigma_n)$ where the base
case is $\theta(q,\varepsilon)  = \theta(q)$.
Furthermore, we write $\tau(\sigma_1 \dots\sigma_n)$ for
$\tau(q_0,\sigma_1 \dots\sigma_n)$,
and we write $\theta(\sigma_1 \dots\sigma_n)$ 
for $\theta(q_0,\sigma_1 \dots\sigma_n)$.
Finally, we call $\theta(\sigma_1 \dots\sigma_n)$ \emph{the output of the
transducer} on $\sigma_1 \dots\sigma_n$.

\paragraph{Probabilistic Automata.}
We follow \cite{balle2013pdfa}.
A \emph{Probabilistic Deterministic Finite Automaton} (PDFA) is a tuple
$\mathcal{A} = \langle Q,\Sigma, \tau, \lambda, \stopaction, q_0 \rangle$ where: 
$Q$ is a finite set of states; 
$\Sigma$ is an arbitrary finite alphabet;
$\tau: Q \times \Sigma \to Q$ is the transition function;
$\lambda: Q \times (\Sigma \cup \{ \stopaction \}) \to [0,1]$
defines the probability of emitting each symbol from each state
($\lambda(q,\sigma)=0$ when $\sigma \in \Sigma$ and $\tau(q,\sigma)$ is not
defined);
$\stopaction$ is a special symbol not in $\Sigma$ reserved to mark the end of a
string; 
$q_0 \in Q$ is the initial state.
It is required that:
\begin{romanenumerate*}
\item
  $\lambda(q,\sigma) = 0$ when $\sigma \in \Sigma$ and $\tau(q,\sigma)$ is not
  defined;
\item
  for each state $q \in Q$, $\sum_{\sigma \in (\Sigma \cup \{\stopaction \})}
  \lambda(q,\sigma) = 1$;
\item
  for each state $q \in Q$ that can be reached from the initial state $q_0$ with
  non-zero probability, there is a state 
  $q' \in Q$ with $\lambda(q',\stopaction) > 0$
  that can be reached from $q$.
\end{romanenumerate*}
The transition function $\tau$ is extended to strings as for transducers, and
the probability function is extended as $\lambda(q,\varepsilon) = 1$ and 
$\lambda(q,\sigma_1\sigma_2 \dots \sigma_n) = \lambda(q,\sigma_1) \cdot 
\lambda(q,\sigma_2 \dots \sigma_n)$.
Then, the probability that $x \in \Sigma^*$ is a string
generated by the automaton starting from state $q$ is $\lambda(q,x\stopaction)$,
and the probability that it is a \emph{prefix} is $\lambda(q,x)$.
For $\mu > 0$,
we say that $\mathcal{A}$ is \emph{$\mu$-distinguishable} if
$\max_x |\lambda(q_1,x)-\lambda(q_2,x)| \geq \mu$ for every
two distinct states $q_1$ and $q_2$. 

\paragraph{Non-Markov Decision Processes.}
A \emph{Non-Markov Decision Process} (NMDP) (cf.\ \cite{brafman2019rdp})
is a tuple 
$\mathcal{P} = \langle A, S, R, \transitionfunc, \rewardfunc, \gamma \rangle$ 
where:
$A$ is a finite set of \emph{actions};
$S$ is a finite set of \emph{states} (or \emph{observation states} to
distinguish them from automata and transducer states);
$R \subseteq \mathbb{R}_{\geq 0}$ is a finite set of non-negative \emph{reward
values};
$\transitionfunc: S^* \times A \times S \to [0,1]$ is the 
\emph{transition function} which defines a probability
distribution $\transitionfunc(\cdot | h,a)$ over $S$ for every $h \in S^*$ and
every $a \in A$;
$\rewardfunc: S^* \times A \times S \to \rewardValues$ is the 
\emph{reward function};  
$\gamma \in (0,1)$ is the \emph{discount factor}.
The transition and reward functions can be combined into the
\emph{dynamics function} $\dynfunc: S^* \times A \times S \times \rewardValues
\to [0,1]$ which defines a probability distribution 
$\dynfunc(\cdot | h,a)$ over $S \times R$ for every $h \in S^*$ and every 
$a \in A$.
Namely, $\dynfunc(s,r|h,a)$ is $\transitionfunc(s|h,a)$ if 
$r = \rewardfunc(h,a,s)$, and zero otherwise.
Every element of $S^*$ is called a \emph{history}.
A \emph{policy} is a function $\pi: S^* \times A \to [0,1]$ that, for every
history $h$, defines a probability distribution $\pi(\cdot|h)$ over the
actions $A$. 
A policy $\pi$ is \emph{deterministic on a history $h$} if $\pi(a|h) = 1$ for
some action $a$, in which case we write $\pi(h) = a$; and it
is \emph{deterministic} if it is deterministic on every history.
A policy $\pi$ is \emph{uniform on a history $h$} if $\pi(a|h) = 1/|A|$ for
every action $a \in A$.
We call $\upolicy$ the policy that is uniform on every history.
Every element of $(AS\rewardValues)^*$ is called a \emph{trace}.
The \emph{dynamics of $\mathcal{P}$ under a policy $\pi$} describe
the probability of an upcoming trace, given the history so far, when actions
are chosen according to a policy $\pi$; it can be recursively computed as
$\dynfunc_\pi(asrt | h) = 
\pi(a|h) \cdot \dynfunc(s,r|h) \cdot \dynfunc_\pi(t | hs)$, with base case
$\dynfunc_\pi(\varepsilon | h) = 1$ for $\varepsilon$ the empty trace.
The \emph{value of a policy $\pi$ on a history $h$},
written $\valfunc_\pi(h)$, is the expected discounted sum of future rewards when
actions are chosen according to $\pi$ given that the history so far is $h$; it
can be recursively computed as
$\valfunc_\pi(h) = \sum_{asr} 
\pi(a|h) \cdot \dynfunc(s,r|h,a) \cdot (r + \gamma \cdot \valfunc_\pi(hs))$.
The \emph{optimal value on a history $h$} is $\valfunc_*(h) = \max_\pi \valfunc_\pi(h)$,
which can be expressed without reference to any policy as 
$\valfunc_*(h) = \max_a \left( \sum_{sr} \dynfunc(s,r|h,a) \cdot (r + \gamma
\cdot \valfunc_*(hs)) \right)$.
The \emph{value of an action $a$ on a history $h$ under a policy $\pi$ },
written $\avalfunc_\pi(h,a)$, is the expected discounted sum of future rewards
when the next action is $a$ and the following actions are chosen according to
$\pi$, given that the history so far is $h$; it is
$\avalfunc_\pi(h,a) = \sum_{sr} 
\dynfunc(s,r|h,a) \cdot (r + \gamma \cdot \valfunc_\pi(hs))$.
The \emph{optimal value of an action $a$ on a history $h$} is $\avalfunc_*(h,a)
= \max_\pi \avalfunc_\pi(h,a)$, and it can be expressed as 
$\avalfunc_*(h,a) = \sum_{sr} 
\dynfunc(s,r|h,a) \cdot (r + \gamma \cdot \valfunc_*(hs))$.
A policy $\pi$ is \emph{optimal on a history $h$} if 
$\valfunc_\pi(h) = \valfunc_*(h)$. For 
$\epsilon > 0$, a policy $\pi$ is \emph{$\epsilon$-optimal on $h$} if
$\valfunc_\pi(h) \geq \valfunc_*(h) - \epsilon$.
A policy is \emph{optimal} (resp., \emph{$\epsilon$-optimal}) if it is so on
every history.
We also say \emph{near-optimal} to say $\epsilon$-optimal for some $\epsilon$.

\paragraph{Regular Decision Processes.}
A \emph{Regular Decision Process (RDP)} \cite{brafman2019rdp}\footnote{In
  \cite{brafman2019rdp} the functions $\transitionfunc$ and
$\rewardfunc$ are represented using the temporal
logics on finite traces $\textsc{ldl}_f$. Here instead we use directly finite
transducers to express them. Note that all  $\transitionfunc$ and
$\rewardfunc$ representable in $\textsc{ldl}_f$ are indeed expressible through finite transducers.} is an NMDP
$\rdp = \langle A, S, R, \transitionfunc, \rewardfunc,
\gamma \rangle$ 
whose transition and reward functions can be
represented by \emph{finite transducers}.
Specifically, there is a finite transducer that,
on every history $h$, outputs the function
$\transitionfunc_h: A \times S \to [0,1]$ induced by $\transitionfunc$
when its first argument is $h$; and there is
a finite transducer that,
on every history $h$, outputs the function
$\rewardfunc_h: A \times S \times R \to [0,1]$ induced by $\rewardfunc$
when its first argument is $h$.
Note that the cross-product of such transducers yields a finite transducer for
the dynamics function $\dynfunc$ of $\rdp$.

\paragraph{Markov Decision Processes.}
A \emph{Markov Decision Process (MDP)}
\cite{bellman1957markovian,puterman1994markov} is an 
NMDP $\langle A, S, R, \transitionfunc, \rewardfunc, \gamma \rangle$ where
both the transition function and the reward function (and hence the dynamics
function) depend only on the last state in the history, which is never empty
since every history starts with a state that is chosen arbitrarily.
Specifically, for every pair of non-empty histories $h_1s$ and $h_2s$,
it holds that $\transitionfunc(h_1s,\cdot) = \transitionfunc(h_2s, \cdot)$ and 
$\rewardfunc(h_1s,\cdot) = \rewardfunc(h_2s, \cdot)$.
Similarly, we say that a policy $\pi$ is \emph{stationary} if
$\pi(h_1s,\cdot) = \pi(h_2s, \cdot)$ for every pair of non-empty histories 
$h_1s$ and $h_2s$.
In the case of MDPs, we can see all history-dependent functions---e.g.,
transition and reward functions, value functions, policies---as
taking a single state in place of a history.

\section{PAC-RL in RDPs}

We consider \emph{reinforcement learning (RL)} as the problem of an agent that
has to learn an optimal policy for an unknown RDP $\rdp$, by acting
and receiving observation states and rewards according to the transition and
reward functions of $\rdp$.
The agent performs a sequence of actions $a_1 \dots a_n$, receiving an
observation state $s_i$ and a reward $r_i$ after each action $a_i$.
The agent has the opportunity to \emph{stop} and possibly start over.
This process generates strings of the form $a_1s_1r_1 \dots
a_ns_nr_n$ which we call \emph{episodes}. They form the \emph{experience} of the
agent, which is the basis to learn an optimal policy.

\begin{example} \label{ex:running-example-1}
  As a running example, we consider an agent that has to cross a $2 \times m$
  grid while avoiding enemies.
  Every enemy guards a two-cell column:
  enemy $i$ guards cells $(0,i)$ and $(1,i)$, and it is found in
  cell $(0,i)$ with probability $p_i^0$ or $p_i^1$. Initially the probability is
  $p_i^0$, and the two probabilities are swapped every time the agent hits
  an enemy.
  At every step, say $i$, the agent has to decide whether to go through cell 
  $(0,i)$ or $(1,i)$, taking action $a_0$ or $a_1$, respectively.
  Specifically, when in cell $(0,i)$ or $(1,i)$, action $a_0$ leads to $(0,i+1)$
  and action $a_1$ to $(1,i+1)$.
  From the last column, the agent is brought back to the first one.
  The agent receives reward one every time it avoids an enemy.
  Each observation state $s$ is a triple
  $\langle i, j, e \rangle$ where $i \in [0,1]$ and $j \in [0,m-1]$ are the
  coordinates of the agent and $e \in \{ \mathit{enemy},
  \mathit{clear} \}$ denotes whether an enemy is in the agent's current cell.
  To maximise rewards, the agent has to learn to predict the enemies'
  positions, which depend on the history of observation states in a regular
  manner.
  Such a dependency is described by a transducer with $2m$
  states of the form $\langle q_1, q_2 \rangle$ where $q_1 \in [0,m-1]$
  stores the agent's current column and $q_2$ is a bit that keeps track of
  whether probabilities $p_i^0$ or $p_i^1$ are being used. 
\end{example}

Our goal is to understand how fast an agent can learn a near-optimal policy.
In particular, we want the agent to find near-optimal policies in the shortest
possible time, and we are not interested in the agent maximising the collected
rewards during learning. 
Thus, we require the agent to return policies 
$\pi_1, \pi_2, \dots$ that improve over time. The idea is that these policies
can be passed to an executing agent, living in such an environment, whose
behaviour will improve as it receives better policies.
Then, the central question is how fast an agent can reach a point after
which it outputs only good policies.

Now we make our discussion more precise.
To measure the learning performance of an agent, we consider the number of
\emph{steps} it performs.
\begin{definition}
  An \emph{action step} consists in performing an action and collecting the
  resulting observation-state and reward.
  A \emph{step} is an action step or an elementary computation step.
\end{definition}
Then, to have a reasonable notion of learning, we borrow from the 
\emph{Probably Approximately Correct (PAC)} framework
\cite{valiant1984theory,kearns1994introduction}, similarly to what is done in
previous work on RL in MDPs
(cf.\ \cite{fiechter1994efficient,kearns2002near,strehl2009reinforcement}).
The PAC framework is based on the observation that exact learning is
infeasible. Thus, it introduces two parameters $\epsilon > 0$ and
$\delta \in (0,1)$ that describe the required accuracy and the confidence of
success, respectively.
In the RL setting it translates into looking for policies that are
$\epsilon$-optimal with probability at least $1-\delta$.
Then, RL is considered feasible if these policies can be found in a number of
steps that is polynomial in $1/\epsilon$ and $\ln(1/\delta)$, and in other
parameters that describe the underlying RDP.
\begin{definition}
  An RL agent (or algorithm) is said to reach accuracy $\epsilon$ and confidence
  $\delta$ in the moment it returns the first policy $\pi_*$ such that
  $\pi_*$ is $\epsilon$-optimal with probability at least $1-\delta$ and the
  same holds for every policy returned after $\pi_*$.
\end{definition}
\begin{definition}
  Let $\vec{d}_\rdp$ be a list of parameters describing $\rdp$,
  let $A$ be its actions, and 
  let $\gamma$ be its discount factor.
  An RL agent (or algorithm) is PAC-RL with respect to $\vec{d}_\rdp$
  if, for every $\epsilon > 0$ and $\delta \in (0,1)$, given
  $(A,\gamma,\epsilon,\delta)$ as input, it reaches accuracy
  $\epsilon$ and confidence $\delta$ in 
  $\operatorname{poly}(1/\epsilon,\ln(1/\delta),\vec{d}_\rdp)$ steps.
\end{definition}



We propose to describe an RDP
$\rdp = \langle A, S, R, \transitionfunc, \rewardfunc, \gamma \rangle$
using the following parameters.
\begin{equation} \label{eq:parameters}
  \vec{d}_\rdp = \left( |A|\text{, } \frac{1}{1-\gamma}\text{, }
    R_\mathrm{max}\text{, } n\text{, } \frac{1}{\rho}\text{, }
    \frac{1}{\mu}\text{, } \frac{1}{\eta} \right)
\end{equation}
The first three parameters take into account the number of actions $|A|$, 
the discount factor $\gamma$, the maximum reward value $R_\mathrm{max}$, and the
number of states $n$ of the minimum transducer for the dynamics of $\rdp$.
Then, the reachability $\rho$ of an RDP measures how easy it is to reach
states of the dynamics transducer $T$ when actions are taken uniformly at
random.
\begin{definition}
  The \emph{reachability} of $\rdp$ is the minimum non-zero probability $\rho$
  that a given state of the dynamics transducer $T$ is reached from the initial
  state within $n$ steps (with $n$ the number of states) when actions are chosen
  uniformly at random.
\end{definition}

The distinguishability $\mu$ of an RDP measures how easy it is to
distinguish states of the dynamics transducer $T$ when actions are taken
uniformly at random.  It is a parameter inherited from the PDFA literature
\cite{ron1998learnability,balle2013pdfa}.
\begin{definition}
  The \emph{distinguishability} of $\rdp$ is the minimum $\mu \in (0,1)$
  such that one of the following conditions holds for every two histories 
  $h_1,h_2$, where $\upolicy$ is the uniform policy:
  \begin{itemize}
    \item 
      $\dynfunc_\upolicy(t|h_1) = \dynfunc_\upolicy(t|h_2)$ for every trace $t$;
    \item 
      $|\dynfunc_\upolicy(t|h_1) - \dynfunc_\upolicy(t|h_2)| > \mu$ for some
      trace $t$.
  \end{itemize}
\end{definition}

The degree of determinism $\eta$ measures how easy it is to discover
transitions. If $\eta$ is small, then there is some transition that is possible,
but unlikely to be observed. Note that this parameter takes value one when the
RDP is deterministic.
\begin{definition}
  The \emph{degree of determinism} $\eta$ of $\rdp$ is the
  minimum non-zero probability that $\transitionfunc(\cdot|h,a)$ assigns
  to an observation state.
\end{definition}

\begin{example}
  Consider the RDP of Example~1.
  Its reachability $\rho$ is $1/2$, since, by an inductive argument,
  every state reachable in $i$ steps has probability at least $1/2$ to be
  visited at step $i$ and, from there, each of the two next states is visited
  with probability $1/2$.
  The degree of determinism $\eta$ is the minimum, for any $i$, among 
  the probabilities $p_i^0$ and $p_i^1$, and their complements.
  The distinguishability $\mu$ is given by the uniform probability of an
  action $1/2$ times the minimum among $\eta$ and the values $|p_i^0-p_i^1|$.
\end{example}


One can show that all the above parameters are necessary.
For the number of actions, discount, and maximum reward parameters
(and also for the accuracy and confidence parameters) the result follows from 
a known lower bound for the MAB problem \cite{mannor2004sample}.
Then, the number of transducer states, reachability, and degree of determinism
can be shown to increase the number of action steps required to visit the
relevant parts of the dynamics transducer a sufficient number of times.
Finally, for the distinguishability parameter, it is straighforward to adapt the
hardness proof for PDFA given in \cite{kearns1994learnability}, which assumes
hardness of learning noisy parity function, a standard cryptographic assumption.
\begin{restatable}{theorem}{thhardness}
  \label{th:hardness}
  There is no algorithm that is PAC-RL with respect to a strict subset
  of the parameters $\vec{d}_\rdp$ given in~\eqref{eq:parameters}, if it is hard
  to learn noisy parity functions.
\end{restatable}

\section{PAC-RL via Probabilistic Automata}
\label{sec:algorithm}

We present an RL algorithm for RDPs that relies on probabilistic
automata.
We consider an RDP
$\rdp = \langle A, S, R, \transitionfunc, \rewardfunc, \gamma \rangle$,
its dynamics function $\dynfunc$, and the minimum transducer 
$T = \langle Q, q_0, S, \tau, \Gamma, \theta \rangle$ that represents
$\dynfunc$ in the sense that $\theta(h)(a,s,r) = \dynfunc(s,r|h,a)$.

\paragraph{Representing RDPs as PDFA.}
When learning, 
the agent has the possibility to experience multiple episodes.
In particular, the agent has an extra action, called 
\emph{stop action}, that ends the current episode and starts a new one.
To take advantage of that,
the agent generates episodes following a stationary policy that
chooses to stop with a non-zero probability $p$, and it chooses an action
from $A$ uniformly if it does not stop.
%
\begin{definition}
  The \emph{stop action} is a special action $\stopaction$ that allows the agent
  to terminate the current episode and start a new one.
  The \emph{exploration policy with stop probability} $p > 0$, written
  $\pi_p$, is the policy that selects the stop action $\stopaction$
  with probability $p$, and each action from $A$ with probability $(1-p)/|A|$.
\end{definition}

Under an exploration policy, an RDP determines a probability distribution on
traces, and hence can be seen as a PDFA.
Specifically, the dynamics of $\rdp$ under $\pi_p$ are captured by the PDFA
$\mathcal{A} = \langle Q, \Sigma, \tau', \lambda, \stopaction, q_0 \rangle$
where:
\begin{itemize}
  \item
    alphabet $\Sigma = \{ asr \in ASR \mid \exists h.\ \dynfunc(s,r|h,a)>0 \}$,
  \item 
    transitions $\tau'(q,asr) = \tau(q,s)$,
  \item 
    probability function:
    \begin{itemize}
      \item
        $\lambda(q,asr) = ((1-p)/|A|) \cdot \theta(q)(a,s,r)$,
      \item
        $\lambda(q,\stopaction) = p$.
    \end{itemize}
\end{itemize}
The dynamics of $\rdp$ are captured in the sense that the following holds for
every history $h$, trace $t$, and string $h'$ such that the projection of
$h'$ on observation states coincides with $h$:
\begin{equation*}
  \dynfunc_{\pi_p}(t|h) = \lambda(\tau'(q_0,h'),t).
\end{equation*}

\begin{algorithm}[t]
  \caption{Reinforcement Learning $\algfont{RL}(A,\gamma,\epsilon,\delta)$}
  \label{alg:rl}
  \textbf{Input}: Actions~$A$, discount factor~$\gamma$,
  required precision~$\epsilon$, confidence parameter~$\delta$.\\
  \textbf{Output}: Policies.
  \begin{algorithmic}[1] 
    \FOR{$\ell = 1, 2, \ldots$}
    \STATE $p \leftarrow 1/(10\ell+1)$;
    $k \leftarrow (2/p) \cdot \ell^2 \cdot (\ell + 5 \ln \ell)$
      \STATE 
        $X \leftarrow \emptyset$;
        $i \leftarrow 0$;
        $\mathit{hardStop} \leftarrow \mathit{false}$
      \WHILE{$\neg \mathit{hardStop}$}
        \STATE  $x, \mathit{hardStop} \leftarrow$ generate an episode under 
        policy $\pi_p$ with a hard stop after $k - i$ actions
        \STATE $i \leftarrow$ increase $i$ by the number of actions in $x$
        \IF{$\neg\mathit{hardStop}$} 
          \STATE $X \leftarrow X \cup \{ x \}$
        \ENDIF
      \ENDWHILE
      \STATE $\hat{\Sigma} \leftarrow$ symbols in $X$;
        $\hatrmax \leftarrow$ max reward in $X$
      \STATE $\hat{\mathcal{A}} \leftarrow$ learn PDFA by calling
      $\adact(\ell, |\hat{\Sigma}|, \delta/2, X)$
      \STATE $\hat{M} \leftarrow$ compute the MDP induced by
      $\hat{\mathcal{A}}$ and $\gamma$
      \STATE $m \leftarrow \left\lceil \frac{1}{1-\gamma} \cdot 
      \ln \left( \frac{2 \cdot \hatrmax}{\epsilon \cdot (1-\gamma)^2}
      \right)\right\rceil$
      \STATE $\pi \leftarrow$ solve $\hat{M}$ by calling
      $\valiter(\hat{M},m)$
      \STATE \textbf{return} transducer for the composition of $\pi$ with the
      projection-on-states of the transition function of $\hat{\mathcal{A}}$
    \ENDFOR
  \end{algorithmic}
\end{algorithm}

\paragraph{PAC-RL Algorithm.}
Algorithm~\ref{alg:rl} provides the pseudocode of our RL algoritm. 
The algorithm repeats the following operations for increasing values of an
integer variable $\ell$, starting from $1$.
(Line~2)~It computes the stop probability $p$, and
the maximum number $k$ of actions to perform during the current iteration.
(Line~3)~It initialises the set of episodes $X$ to the empty set, a counter $i$
for the actions to zero, and a flag $\mathit{hardStop}$ to know when a hard
stop has occurred.
(Lines~4--10)~It generates episodes following the exploration policy with stop
probability $p$, and with a hard stop if the number of performed actions reaches
$k$. Episodes are stored in variable $X$.
(Line~11)~It reads the set $\hat{\Sigma}$ of action-state-reward symbols from
$X$, and the maximum reward $\hatrmax$ occurring in $X$.
(Line~12)~It learns a PDFA 
$\hat{\mathcal{A}} = \langle
\hat{Q},\hat{\Sigma}, \hat{\tau}', \hat{\lambda}, \stopaction, \hat{q}_0
\rangle$ via the $\adact$ algorithm \cite{balle2013pdfa} instantiated with 
$\ell$ as an upper bound on the number of automaton states,
$|\hat{\Sigma}|$ as an estimate of the size of the alphabet, 
confidence parameter $\delta/2$, and set of strings $X$.
(Line~13)~Starting from $\hat{\mathcal{A}}$, the algorithm computes the MDP 
$\hat{M} = \langle A, \hat{Q}, \hat{R}, \dynfunc^{\hat{M}}, \gamma, \hat{q}_0
\rangle$ where $\hat{R}$ consists of each reward value occurring as the third
component of an element of $\hat{\Sigma}$, and the dynamics function is as
follows:
\begin{equation*}
  \dynfunc^{\hat{M}}(q_2,r | q_1,a) = (|A|/(1-p))
    \textstyle\sum_{s:\hat{\tau}'(q_1,asr) = q_2} \hat{\lambda}(q_1,asr).
\end{equation*}
(Lines~14--15)~The algorithm then solves $\hat{M}$ via $m$ iterations of the
classic \emph{value iteration} algorithm with action-value
estimates initialised to zero.
Specifically, value iteration computes an approximation
$\hat{\avalfunc}^{\hat{M}}_*$ of the
optimal action-value function $\avalfunc^{\hat{M}}_*$ of $\hat{M}$, and then
computes the greedy policy $\pi$ with respect $\hat{\avalfunc}^{\hat{M}}_*$,
which is a stationary policy on state space $\hat{Q}$.
(Line~16)~The algorithm returns a transducer that represents the policy function
$\pi(\hat{\tau}(\cdot))$ obtained by composing the stationary policy
$\pi$ and the `projection' transition function $\hat{\tau}$ defined as
$\hat{\tau}(q,s) = \hat{\tau}'(q,asr)$ for an arbitrary choice of $a$ and $r$
such that $\hat{\lambda}(q,asr) > 0$; note that every choice yields the
same result and a choice always exists.
Specifically, the returned transducer is 
$\langle \hat{Q}, \hat{q}_0, \hat{S}, \hat{\tau}, A, \theta' \rangle$ where 
$\theta'(\hat{q}) = \pi(\hat{q})$ and
$\hat{S}$ consists of each observation-state occurring as the second component
of an element of $\hat{\Sigma}$.

\section{PAC Analysis}
\label{sec:pac}

Algorithm~\ref{alg:rl} shows that RL in RDPs is feasible in polynomial time.
\begin{restatable}{theorem}{thpac} 
  \label{th:pac}
  Algorithm~1 is PAC-RL 
  with respect to the parameters $\vec{d}_\rdp$ given in
  \eqref{eq:parameters}.
\end{restatable}

We analyse Algorithm~\ref{alg:rl} to show that the theorem holds.
We first show that we can compute a near-optimal policy for $\rdp$ starting from
a near-optimal policy for the MDP induced by the dynamics transducer $T$, that
is defined as $M = \langle A, Q, R, \dynfunc^M, \gamma \rangle$ where the
dynamics function is as follows:
\begin{equation*}
    \dynfunc^M(q_2,r| q_1,a) = 
    \textstyle \sum_{s:\tau(q_1,s) = q_2} \theta(q_1)(a,s,r).
\end{equation*}
We call it the \emph{ideal MDP}, since it is the MDP that the algorithm would
build if it learned the automaton $\hat{\mathcal{A}}$ perfectly. Its 
key property is that an $\epsilon$-optimal policy for $\rdp$ can be
obtained from an $\epsilon$-optimal policy for $M$ by composing it with the
transition function of $T$.
\begin{restatable}{lemma}{lemmastationaryoptimalpolicy}
  \label{lemma:stationary-optimal-policy}
  If $\pi$ is an $\epsilon$-optimal policy for $M$, then
  $\pi(\tau(\cdot))$ is an $\epsilon$-optimal policy for $\rdp$.
\end{restatable}

\begin{table}
  \centering
  \begin{tabular}{ll}
    \toprule
    & Definition \\
    \midrule
    $\epsm$ & $\frac{(1-\gamma)^3 \cdot \epsilon}{3 \cdot
    R_\mathrm{max}}$
    \\[4pt]
    $\epsa$ & $\frac{(1-p) \cdot \epsm}{|A| \cdot n \cdot |\Sigma|}$
    \\[4pt]
    $\delta_0$ & $\frac{\delta}{2 \cdot \hat{n} \cdot (\hat{n} \cdot |\Sigma|
      + |\Sigma| + 1)}$
    \\[5pt]
    $\samplesone$ & $\frac{22 \cdot \euler \cdot |A| \cdot \hat{n} \cdot
    |\Sigma|}{\rho \cdot \eta \cdot (1-p) \cdot \mu^2} \cdot \ln \frac{704 \cdot
    |A| \cdot \hat{n} \cdot
    |\Sigma|}{\rho \cdot \eta \cdot \mu^2 \cdot p \cdot \delta_0^2}$
    \\[5pt]
    $\samplestwo$ & $\frac{\hat{n} \cdot |\Sigma| \cdot |A|}{0.9 \cdot \rho
    \cdot \eta \cdot (1-p) \cdot (\epsa)^2} \cdot \ln \frac{2 \cdot (|\Sigma| +
    1)}{\delta_0}$
    \\
    \bottomrule
  \end{tabular}
  \caption{Quantities used in the PAC analysis.}
  \label{tab:quantities}
\end{table}

Now that we know that it suffices to compute an $\epsilon$-optimal policy for
the ideal MDP $M$, we establish the accuracy required for the MDP $\hat{M}$ that
is actually computed by the algorithm (Line~14) in order to be the basis for
computing an $\epsilon$-optimal policy for $M$.
\begin{definition} \label{def:mdp-approximation}
  We say that $\hat{M}$ is an $\alpha$-approximation of $M$ if the two MDPs 
  have the same actions, rewards, and discount factor, and there is a bijection
  $\phi$ between their states such that 
  $\| \dynfunc^M(\cdot|q,a) - \dynfunc^{\hat{M}}(\cdot|\phi(q),a) \|_1 \leq
  \alpha$ for every $q$ and $a$.
\end{definition}
The following lemma is an application of known results for MDPs.
First, if $\hat{M}$ approximates $M$, then the value functions of $\hat{M}$
approximate the value functions of $M$ \cite{strehl2005theoretical}.
Then, the value functions of $\hat{M}$ can be computed in polynomial time with
high accuracy via the value iteration algorithm \cite{strehl2009reinforcement}.
Finally, any policy that greedily chooses the best action according to an
accurate estimate of the optimal action-value function is near-optimal
\cite{singh1994upper}.
\begin{restatable}{lemma}{lemmaapproximationofmdp}
  \label{lemma:approximation-of-mdp}
  If $\hat{M}$ is an $\epsm$-approximation of $M$ (with $\epsm$ as in
  Table~\ref{tab:quantities}), then 
  the greedy policy obtained via 
  $$\left\lceil \frac{1}{1-\gamma} \cdot 
    \ln \left( \frac{2 \cdot \rmax}{\epsilon \cdot (1-\gamma)^2}
  \right)\right\rceil$$
  iterations of the \emph{value iteration} algorithm, with action-value
  estimates initialised to zero, is an $\epsilon$-optimal policy for $M$.
\end{restatable}
We establish the accuracy required for the learned PDFA.
\begin{definition} \label{def:pdfa-approximation}
  We say that $\hat{\mathcal{A}}$ is an $\alpha$-approximation of
  $\mathcal{A}$ if the two automata have the same alphabet, there is a
  transition-preserving isomorphism $\phi$ between their states, and their
  probability functions satisfy
  $|\lambda(q,\sigma) - \hat{\lambda}(\phi(q),\sigma)| < \alpha$
  for every state $q$ and symbol $\sigma$, with the additional condition that
  $\lambda(q,\sigma) = 0$ implies $\hat{\lambda}(\phi(q),\sigma) = 0$.
\end{definition}
The bound on the error for the learned automaton transfers to the MDP it
induces, amplified by the number of transducer states, alphabet size, 
and the inverse of the minimum probability assigned to an action by the
exploration policy $\pi_p$.
\begin{restatable}{lemma}{lemmaapproximationofpdfa}
  \label{lemma:approximation-of-pdfa}
  If $\hat{\mathcal{A}}$ is an $\epsa$-approximation of $\mathcal{A}$, 
  then $\hat{M}$ is an $\epsm$-approximation of $M$ (with $\epsa$ as
  in Table~\ref{tab:quantities}).
\end{restatable}
We next derive the number of episodes and the stop probability (which determines
the length of episodes) under which the $\adact$ algorithm guarantees to learn 
$\epsa$-approximation of $\mathcal{A}$.
The accuracy of the algorithm depends on the distinguishability of the automaton
to learn, on the minimum probability that a state is visited during a run, and
on the minimum transition probability.
%
Thus, we need establish lower bounds for the former three quantities.
%
%
First, the distinguishability $\mathcal{A}$ is close to the distinguishability
$\mu$ of $\rdp$ if $p$ is sufficiently small, since states can be distinguished
by looking at strings of length at most the number of states $n$. In fact, for
any pair of longer strings witnessing the distinguishability, we can take their
substrings of length at most $n$ obtained by removing `cycles', similarly to
what is done in the \emph{pumping lemma} for regular languages.
The bound then follows by taking into account the probability $(1- p)^n$ of
generating a string of length at least $n$.
Second, the probability of visiting a state of $\mathcal{A}$ is at least
the probability of generating a string of length at least $n$ times
the probability of visting the corresponding state in $T$ while generating the
string, which is at least the reachability $\rho$ of $\rdp$.
Third, the minimum transition probability in $\mathcal{A}$ is at least the
degree of determinism $\eta$ of $\rdp$ times the probability of picking an
action uniformly when not stopping.
\begin{restatable}{lemma}{lemmadistinguishability}
  \label{lemma:distinguishability}
  The distinguishability of $\mathcal{A}$ is at least $\mu \cdot (1- p)^n$.
  The minimum non-zero probability that a state of $\mathcal{A}$ is visited
  during a run is at least $\rho \cdot (1-p)^n$.
  The minimum non-zero probability of a non-$\stopaction$ transition of
  $\mathcal{A}$ is at least $\eta \cdot (1-p)/|A|$.
\end{restatable}
Given the bounds above, the next lemma follows from the
guarantees for the $\adact$ algorithm \cite{balle2013pdfa}. 
In particular, the guarantees require a number of episodes that depends on the
specified upper bound on the states, on the specified alphabet size, on the
required accuracy and confidence, and on the three quantities of the previous
lemma, which applies considering that $(1-p)^n \geq 0.9$ because of the
condition on $p$.
\begin{restatable}{lemma}{lemmanumberofepisodes}
  \label{lemma:number-of-episodes}
  If $p \leq 1/(10  n +1)$, $\hat{n}$ is an upper bound on the
  number $n$ of states of $\mathcal{A}$, and 
  $X$ are strings generated by $\mathcal{A}$ with $|X| \geq
  \max(\samplesone,\samplestwo)$,
  then $\adact(\hat{n}, |\Sigma|, \delta/2, X)$ returns an
  $\epsa$-approximation of $\mathcal{A}$ with probability 
  $1-\delta/2$.
\end{restatable}
The required sample size and stop probability are achieved in a polynomial
number $\ell$ of iterations of the algorithm. The sample size also guarantees 
$\hat{\Sigma} = \Sigma$ and $\hatrmax = \rmax$.  
Furthermore, introducing a logarithmic dependency on $\delta$
ensures that a hard stop occurs with probability at most $\delta/2$, which
combined with the probability of failure of $\adact$ is still less than
$\delta$.
The $\ell$-th iteration of the algorithm performs the number of actions $k$
specified in Line~2, which is $\widetilde{O}(\ell^4)$, and hence
the algorithm performs $\widetilde{O}(\ell^5)$ action steps in
$\ell$ iterations, which is:
\begin{equation*}
\widetilde{O}\left(n^5 + \frac{|\Sigma|^5 \cdot |A|^5}{\rho^5 \cdot \eta^5}
  \cdot \left(\frac{1}{\mu^{10}} +
    \frac{n^{10} \cdot |\Sigma|^{10} \cdot |A|^{10} \cdot
    R_\mathrm{max}^{10}}{\epsilon^{10}
  \cdot (1-\gamma)^{30}}\right)\right)
\end{equation*}
The bound mentions $|\Sigma|$ which is not in $\vec{d}_\rdp$, but satisfies 
$|\Sigma| \leq n \cdot |A| \cdot \lceil 1/\eta \rceil$.
Then, a polynomial number of action steps immediately implies a polynomial
number of steps overall.
In particular, $\adact$ runs in time polynomial in the size of the input sample
and in the specified values for the number of states and alphabet size, and also
value iteration $\valiter$ runs in time polynomial in the size of the input MDP
and in the number of iterations. We conclude that Algorithm~\ref{alg:rl} reaches
accuracy $\epsilon$ and confidence $\delta$ in a polynomial number of steps.  
Therefore, Algorithm~\ref{alg:rl} is PAC-RL.

\section{Exploiting Prior Knowledge}

We observe that if we know (or we can estimate) the number of states in the
dynamics transducer, we can devise a simpler algorithm,
Algorithm~\ref{alg:na-rl}, with better performance bounds.
Algorithm~\ref{alg:na-rl} is simpler than Algorithm~\ref{alg:rl}, since
it does need to search for the number of transducer states.
The stop probability is constant throughout a run, and computed based on the
upper bound on the states.
Since the stop probability is uniform across iterations, all episodes can be
seen as generated by the same automaton, and hence we can
accumulate episodes instead of deleting previous ones.
Furthermore, we can now call $\adact$ directly with the given parameter 
$\hat{n}$.
A simplified analysis, again based on
Lemmas~\ref{lemma:stationary-optimal-policy}--\ref{lemma:number-of-episodes},
yields the following polynomial bound on the expected number of action steps,
which immediately implies a polynomial bound on the expected number of steps
overall.
\begin{restatable}{theorem}{thadditionalalg}
  \label{th:additional-alg}
  Algorithm~\ref{alg:na-rl} on input 
  $(A,\gamma,\epsilon,\delta,\hat{n})$ reaches 
  accuracy $\epsilon$ and confidence $\delta$ within an expected number of
  action steps that is:
  $$\widetilde{O}\left(\frac{|A| \cdot \hat{n} \cdot
    |\Sigma|}{\rho \cdot \eta}
  \cdot \left( \frac{1}{\mu^2} + \frac{\hat{n}^2 \cdot |\Sigma|^2 \cdot |A|^2
    \cdot \rmax^2}{(1-\gamma)^6 \cdot \epsilon^2} \right)\right).$$
\end{restatable}

\begin{algorithm}[t]
  \caption{Reinforcement Learning
  $\algfont{RL}(A,\gamma,\epsilon,\delta,\hat{n})$}
  \label{alg:na-rl}
  \textbf{Input}: Actions~$A$, discount factor~$\gamma$,
  required precision~$\epsilon$, confidence parameter~$\delta$,
  upper bound $\hat{n}$ on transducer states.\\
  \textbf{Output}: Policies.
  \begin{algorithmic}[1] 
    \STATE $p \leftarrow 1/(10 \cdot \hat{n}+1)$ 
    \STATE $X \leftarrow \emptyset$
    \LOOP
      \STATE $x \leftarrow$ generate an episode under exploration policy $\pi_p$
      \STATE $X \leftarrow X \cup \{ x \}$
      \STATE $\hat{\Sigma} \leftarrow$ symbols in $X$;
        $\hatrmax \leftarrow$ max reward in $X$
      \STATE $\hat{\mathcal{A}} \leftarrow$ learn PDFA by calling
      $\adact(\hat{n}, |\hat{\Sigma}|, \delta, X)$
      \STATE $\hat{M} \leftarrow$ compute the MDP induced by $\hat{\mathcal{A}}$
      and $\gamma$
      \STATE $m \leftarrow \left\lceil \frac{1}{1-\gamma} \cdot 
      \ln \left( \frac{2 \cdot \hatrmax}{\epsilon \cdot (1-\gamma)^2}
      \right)\right\rceil$
      \STATE $\pi \leftarrow$ solve $\hat{M}$ by calling
      $\valiter(\hat{M},m)$
      \STATE \textbf{return} transducer for the composition of $\pi$ with the
      projection-on-states of the transition function of $\hat{\mathcal{A}}$
    \ENDLOOP
  \end{algorithmic}
\end{algorithm}

\section{Discussion}

We have presented RL algorithms that can learn near-optimal policies for RDPs in
polynomially-many steps, in the parameters that describe the underlying RDP.

\begin{example}
  Algorithm~1 and Algorithm~2 (with a bound on the number of states that is
  polynomial in the grid length) compute a near-optimal policy in each of
  the RDPs introduced in Example~1 in a number of steps that is polynomial in
  the following quantities: 
  \begin{romanenumerate*}
  \item the grid length $m$,
  \item the inverse of the minimum among $p_i^0$, $p_i^1$, $1-p_i^0$, $1-p_i^1$,
  \item the inverse of the minimum value $|p_i^0-p_i^1|$.
  \end{romanenumerate*}
\end{example}

Adopting PDFA techniques takes us into a different direction from
existing approaches based on a direct clustering of histories such as 
\cite{abadi2020learning}.
There, histories are clustered according to the probability of the
following observation state. Since the algorithm compares single histories, the
accuracy of the algorithm depends on the probability of single histories, which
can be exponentially-low in their length.  
In turn, the required length of histories can grow with 
the number of transducer states, and hence the approach can require
exponentially-many episodes in order to achieve high accuracy. 
For instance, in our running example, a history of length $m$ has probability at
most $g^m$ under any policy, with $g$ the maximum probability among 
$p_i^0, p_u^1,1-p_i^0,1-p_i^1$.
Histories of length $m-1$ are necessary to determine the best action at the
$m$-th step.
To address the issue, PDFA algorithms build
states incrementally while relying on their distinguishability. 
This way,
each state gathers the probability of all the histories it represents.
In light of our results, we believe that these PDFA techniques will be
instrumental in developing the next generation of tools for RL in RDPs.


%
%

%
%
%

\section*{Acknowledgments}
This work has been partially supported by the ERC Advanced Grant WhiteMech (No.\
834228) and by the EU ICT-48 2020 project TAILOR (No.\ 952215).

\bibliographystyle{named}
\bibliography{journal-abbreviations,bibliography}

\begin{thebibliography}{}

\bibitem[\protect\citeauthoryear{Abadi and Brafman}{2020}]{abadi2020learning}
Eden Abadi and Ronen~I. Brafman.
\newblock Learning and solving regular decision processes.
\newblock In {\em IJCAI}, 2020.

\bibitem[\protect\citeauthoryear{Audibert \bgroup \em et al.\egroup
  }{2009}]{audibert2009exploration}
Jean{-}Yves Audibert, R{\'{e}}mi Munos, and Csaba Szepesv{\'{a}}ri.
\newblock Exploration-exploitation tradeoff using variance estimates in
  multi-armed bandits.
\newblock {\em Theor. Comput. Sci.}, 410(19), 2009.

\bibitem[\protect\citeauthoryear{Auer and Ortner}{2006}]{auer2006ucrl}
Peter Auer and Ronald Ortner.
\newblock Logarithmic online regret bounds for undiscounted reinforcement
  learning.
\newblock In {\em NeurIPS}, 2006.

\bibitem[\protect\citeauthoryear{Auer \bgroup \em et al.\egroup
  }{2002}]{auer2002finite}
Peter Auer, Nicol{\`{o}} Cesa{-}Bianchi, and Paul Fischer.
\newblock Finite-time analysis of the multiarmed bandit problem.
\newblock {\em Mach. Learn.}, 47(2--3), 2002.

\bibitem[\protect\citeauthoryear{Azar \bgroup \em et al.\egroup
  }{2017}]{azar2017minimax}
Mohammad~Gheshlaghi Azar, Ian Osband, and R{\'{e}}mi Munos.
\newblock Minimax regret bounds for reinforcement learning.
\newblock In {\em ICML}, volume~70, 2017.

\bibitem[\protect\citeauthoryear{Bai \bgroup \em et al.\egroup
  }{2020}]{bai2020near}
Yu~Bai, Chi Jin, and Tiancheng Yu.
\newblock Near-optimal reinforcement learning with self-play.
\newblock In {\em NeurIPS}, 2020.

\bibitem[\protect\citeauthoryear{Balle \bgroup \em et al.\egroup
  }{2013}]{balle2013pdfa}
Borja Balle, Jorge Castro, and Ricard Gavald{\`{a}}.
\newblock Learning probabilistic automata: {A} study in state
  distinguishability.
\newblock {\em Theor. Comput. Sci.}, 473, 2013.

\bibitem[\protect\citeauthoryear{Balle \bgroup \em et al.\egroup
  }{2014}]{balle2014adaptively}
Borja Balle, Jorge Castro, and Ricard Gavald{\`{a}}.
\newblock Adaptively learning probabilistic deterministic automata from data
  streams.
\newblock {\em Mach. Learn.}, 96(1-2), 2014.

\bibitem[\protect\citeauthoryear{Bellman}{1957}]{bellman1957markovian}
Richard Bellman.
\newblock A {M}arkovian decision process.
\newblock {\em J. Math. Mech.}, 1957.

\bibitem[\protect\citeauthoryear{Brafman and {De
  Giacomo}}{2019}]{brafman2019rdp}
Ronen~I. Brafman and Giuseppe {De Giacomo}.
\newblock Regular decision processes: {A} model for non-{M}arkovian domains.
\newblock In {\em IJCAI}, 2019.

\bibitem[\protect\citeauthoryear{Brafman and
  Tennenholtz}{2002}]{brafman2002rmax}
Ronen~I. Brafman and Moshe Tennenholtz.
\newblock {R-MAX}: {A} general polynomial time algorithm for near-optimal
  reinforcement learning.
\newblock {\em J. Mach. Learn. Res.}, 3, 2002.

\bibitem[\protect\citeauthoryear{Clark and Thollard}{2004}]{clark2004pac}
Alexander Clark and Franck Thollard.
\newblock {PAC}-learnability of probabilistic deterministic finite state
  automata.
\newblock {\em J. Mach. Learn. Res.}, 5, 2004.

\bibitem[\protect\citeauthoryear{Dann and Brunskill}{2015}]{dann2015sample}
Christoph Dann and Emma Brunskill.
\newblock Sample complexity of episodic fixed-horizon reinforcement learning.
\newblock In {\em NeurIPS}, 2015.

\bibitem[\protect\citeauthoryear{Dann \bgroup \em et al.\egroup
  }{2017}]{dann2017unifying}
Christoph Dann, Tor Lattimore, and Emma Brunskill.
\newblock Unifying {PAC} and regret: Uniform {PAC} bounds for episodic
  reinforcement learning.
\newblock In {\em NeurIPS}, 2017.

\bibitem[\protect\citeauthoryear{Fiechter}{1994}]{fiechter1994efficient}
Claude{-}Nicolas Fiechter.
\newblock Efficient reinforcement learning.
\newblock In {\em COLT}, 1994.

\bibitem[\protect\citeauthoryear{Jaksch \bgroup \em et al.\egroup
  }{2010}]{jaksch2010near}
Thomas Jaksch, Ronald Ortner, and Peter Auer.
\newblock Near-optimal regret bounds for reinforcement learning.
\newblock {\em J. Mach. Learn. Res.}, 11, 2010.

\bibitem[\protect\citeauthoryear{Kakade}{2003}]{kakade2003thesis}
Sham~M. Kakade.
\newblock {\em On the Sample Complexity of Reinforcement Learning}.
\newblock PhD thesis, Gatsby Unit, University College London, 2003.

\bibitem[\protect\citeauthoryear{Kearns and Singh}{2002}]{kearns2002near}
Michael~J. Kearns and Satinder~P. Singh.
\newblock Near-optimal reinforcement learning in polynomial time.
\newblock {\em Mach. Learn.}, 49(2--3), 2002.

\bibitem[\protect\citeauthoryear{Kearns and
  Vazirani}{1994}]{kearns1994introduction}
Michael~J. Kearns and Umesh~V. Vazirani.
\newblock {\em An Introduction to Computational Learning Theory}.
\newblock {MIT} Press, 1994.

\bibitem[\protect\citeauthoryear{Kearns \bgroup \em et al.\egroup
  }{1994}]{kearns1994learnability}
Michael~J. Kearns, Yishay Mansour, Dana Ron, Ronitt Rubinfeld, Robert~E.
  Schapire, and Linda Sellie.
\newblock On the learnability of discrete distributions.
\newblock In {\em STOC}, 1994.

\bibitem[\protect\citeauthoryear{Kearns \bgroup \em et al.\egroup
  }{2002}]{kearns2002sparse}
Michael~J. Kearns, Yishay Mansour, and Andrew~Y. Ng.
\newblock A sparse sampling algorithm for near-optimal planning in large
  {M}arkov decision processes.
\newblock {\em Mach. Learn.}, 49(2-3), 2002.

\bibitem[\protect\citeauthoryear{Lattimore and
  Hutter}{2014}]{lattimore2014near}
Tor Lattimore and Marcus Hutter.
\newblock Near-optimal {PAC} bounds for discounted {MDP}s.
\newblock {\em Theor. Comput. Sci.}, 558, 2014.

\bibitem[\protect\citeauthoryear{Mannor and
  Tsitsiklis}{2004}]{mannor2004sample}
Shie Mannor and John~N. Tsitsiklis.
\newblock The sample complexity of exploration in the multi-armed bandit
  problem.
\newblock {\em J. Mach. Learn. Res.}, 5, 2004.

\bibitem[\protect\citeauthoryear{Moore}{1956}]{moore1956gedanken}
Edward~F. Moore.
\newblock Gedanken-experiments on sequential machines.
\newblock {\em Automata Studies}, 34, 1956.

\bibitem[\protect\citeauthoryear{Palmer and Goldberg}{2007}]{palmer2007pac}
Nick Palmer and Paul~W. Goldberg.
\newblock {PAC}-learnability of probabilistic deterministic finite state
  automata in terms of variation distance.
\newblock {\em Theor. Comput. Sci.}, 387(1), 2007.

\bibitem[\protect\citeauthoryear{Puterman}{1994}]{puterman1994markov}
Martin~L. Puterman.
\newblock {\em {M}arkov Decision Processes: Discrete stochastic dynamic
  programming}.
\newblock Wiley, 1994.

\bibitem[\protect\citeauthoryear{Ron \bgroup \em et al.\egroup
  }{1998}]{ron1998learnability}
Dana Ron, Yoram Singer, and Naftali Tishby.
\newblock On the learnability and usage of acyclic probabilistic finite
  automata.
\newblock {\em J. Comput. Syst. Sci.}, 56(2), 1998.

\bibitem[\protect\citeauthoryear{Singh and Yee}{1994}]{singh1994upper}
Satinder~P. Singh and Richard~C. Yee.
\newblock An upper bound on the loss from approximate optimal-value functions.
\newblock {\em Mach. Learn.}, 16(3), 1994.

\bibitem[\protect\citeauthoryear{Strehl and
  Littman}{2005}]{strehl2005theoretical}
Alexander~L. Strehl and Michael~L. Littman.
\newblock A theoretical analysis of model-based interval estimation.
\newblock In {\em ICML}, 2005.

\bibitem[\protect\citeauthoryear{Strehl \bgroup \em et al.\egroup
  }{2009}]{strehl2009reinforcement}
Alexander~L. Strehl, Lihong Li, and Michael~L. Littman.
\newblock Reinforcement learning in finite {MDP}s: {PAC} analysis.
\newblock {\em J. Mach. Learn. Res.}, 10, 2009.

\bibitem[\protect\citeauthoryear{Sutton and Barto}{2018}]{suttonbarto}
Richard~S. Sutton and Andrew~G. Barto.
\newblock {\em Reinforcement learning}.
\newblock {MIT} Press, 2018.

\bibitem[\protect\citeauthoryear{Szita and
  Szepesv{\'{a}}ri}{2010}]{szita2010mormax}
Istvan Szita and Csaba Szepesv{\'{a}}ri.
\newblock Model-based reinforcement learning with nearly tight exploration
  complexity bounds.
\newblock In {\em ICML}, 2010.

\bibitem[\protect\citeauthoryear{Valiant}{1984}]{valiant1984theory}
Leslie~G. Valiant.
\newblock A theory of the learnable.
\newblock {\em Commun. {ACM}}, 27(11), 1984.

\bibitem[\protect\citeauthoryear{Wang \bgroup \em et al.\egroup
  }{2020}]{wang2020long}
Ruosong Wang, Simon~S. Du, Lin~F. Yang, and Sham~M. Kakade.
\newblock Is long horizon {RL} more difficult than short horizon {RL}?
\newblock In {\em NeurIPS}, 2020.

\end{thebibliography}

\ifextendedversion
  \newpage
\appendix
\onecolumn 
\section*{Appendix}
The appendix is organised as follows:
\begin{itemize}
  \item
    Appendix~\ref{sec:proofs} contains proofs of all our technical results.
  \item
    Appendix~\ref{sec:running-example} contains a detailed description of
    our running example.
  \item
    Appendix~\ref{sec:comparison} contains a comparison of our approach with the
    one of \cite{abadi2020learning}.
\end{itemize}

\section{Proofs}
\label{sec:proofs}

\subsection{Proof of Theorem~\ref{th:hardness}}

We prove Theorem~1 through a series of lemmas, with each lemma showing
the necessity of one of the parameters given in Equation~\eqref{eq:parameters},
which we restate here for convenience:
\begin{equation*}
  \vec{d}_\rdp = \left( |A|\text{, } \frac{1}{1-\gamma}\text{, }
    R_\mathrm{max}\text{, } n\text{, } \frac{1}{\rho}\text{, }
    \frac{1}{\mu}\text{, } \frac{1}{\eta} \right).
\end{equation*}


The following lemma shows the necessity of the \emph{distinguishability}
parameter $\mu$. 
The result is based on a construction from \cite{kearns2002sparse} and it relies
on the conjecture that noisy parity functions are difficult to learn. Briefly,
it is difficult to learn a parity function from input-output pairs if the inputs
are chosen uniformly at random and the outputs are flipped with probability
$\alpha \in (0,1/2)$, called the noise rate. A clear
statement of the conjecture can be found in \cite{kearns2002sparse}.
\begin{lemma}
  Under the assumption that noisy parity functions cannot be PAC-learned in
  polynomial time, for every algorithm there is an RDP $\rdp$ such that the
  algorithm does
  not reach accuracy $\epsilon$ and confidence $\delta$ in a number of steps
  $\operatorname{poly}(1/\epsilon,\ln (1/\delta),\vec{d}_\rdp')$ if 
  $\vec{d}_\rdp'$ is $\vec{d}_\rdp$ after removing the distinguishability
  parameter.
\end{lemma}
\begin{proof}
  We show a class of RDPs such that the existence of an algorithm that reaches
  accuracy $\epsilon$ and confidence $\delta$ in the claimed number of actions
  steps in all such RDPs contradicts the conjecture about noisy parity
  functions.
  These RDPs are inspired by the construction given in Theorem~16 of 
  \cite{kearns1994learnability}.
  Consider a parity function $f_S: \{ 0,1 \}^m \to \{ 0,1 \}$ where 
  $S \subseteq \{ x_1, \dots, x_m \}$ and
  $f_S(x_1, \dots, x_m) = 1$ iff the parity of $\vec{x} = x_1, \dots, x_m$ on
  the set $S$ is $1$.
  We then build an RDP $\rdp_S$ where learning a $1$-optimal policy amounts to
  learning $f_S$ from noisy samples of $f_S$ for a given noise rate 
  $\alpha \in (0,1/2)$, which contradicts the noisy parity conjecture.
  The RDP $\rdp_S$ has observation states $S = \{ 0,1 \}$ and actions $A = \{
  a_0, a_1 \}$.
  The initial observation state is $0$.
  When the history of observation states has length at most $m+1$, then the
  next observation state is chosen uniformly at random regardless of the action,
  and no reward is issued. 
  When the history $h$ of observation states has length $m+2$, 
  the next observation state is $f_S(h)$ with probability $1-\eta$, and the
  reward is $1$ if the chosen action is $a_b$ with $b = f_S(h)$---i.e., the
  agent has guessed the value of the function correctly.
  From this point onwards, the observation state is always $0$ and no reward is
  issued.
  Note that the transducer of the dynamics described above has states $Q = \{
  q_0, q_1^0, q_1^1, \dots, q_m^0, q_m^1, q_{m+1} \}$, it is initially in state 
  $q_0$, it is in state $q_i^b$ when the history has length $i+1$ and
  the parity of the bits specified by $S$ is $b$, and it is in state 
  $q_{m+1}$ when the history has length greater than $m+1$.
  The output in states $\{ q_0, q_1^0, q_1^1, \dots, q_{m-1}^0, q_{m-1}^1 \}$
  specifies uniform probability over the observation states regardless the
  action, and reward zero.
  The output in state $q_m^b$ specifies probability $1-\eta$ for observation
  $b$, and reward $1$ if the action is $a_b$ regardless of the observation.
  The output in state $q_{m+1}$ specifies probability $1$ for observation $0$
  and reward zero.

  If an algorithm could reach accuracy $1$ and confidence $\delta$ in
  $N=\operatorname{poly}(1/\epsilon,\ln (1/\delta),\vec{d}_\rdp')$ action
  steps, then the learned policy would encode $f_S$ with confidence $\delta$, 
  which contradicts the noisy parity conjecture because $N$ is polynomial---the
  only parameter in $\vec{d}_\rdp'$ to grow is the number of states, which grows
  linearly with $m$.
\end{proof}

The following lemma shows that the required number of action steps increases
with the inverse of the \emph{degree of determinism}.
In other words, the less likely transitions are, the more the agent has to
explore.
This is due to the fact that an agent that has not experienced some 
transition of the dynamics transducer will not be able to determine the
resulting transducer state when the transition occurs.
\begin{lemma}
  For every algorithm there is an RDP $\rdp$ such that the algorithm does not
  necessarily reach accuracy $\epsilon$ and confidence $\delta$ in a number of
  steps
  $\operatorname{poly}(1/\epsilon,\ln (1/\delta),\vec{d}_\rdp')$ if 
  $\vec{d}_\rdp'$ is $\vec{d}_\rdp$ after removing the parameter for the degree
  of determinism.
\end{lemma}
\begin{proof}
  Consider $i \in \{ 1,2 \}$, $\eta \in (0,1)$, and an RDP $\rdp^\eta_i$ with
  two observation states $S = \{ s_1, s_2, s_3 \}$, two actions 
  $A = \{a_1,a_2\}$, and dynamics as described next.
  The initial observation state is $s_1$.
  When the last observation state is $s_1$, both actions yield observation state 
  $s_2$ with probability $1-\eta$ and observation state $s_3$ with probability
  $\eta$.
  When the last observation state is $s_2$, both actions yield observation state 
  $s_1$ with probability $1-\eta$ and observation state $s_3$ with probability
  $\eta$.
  When the last observation state is $s_3$, both actions yield observation state 
  $s_i$ with probability $1-\eta$ and observation state $s_3$ with probability
  $\eta$.
  Note that the dynamics of both RDPs are represented by a transducer with two
  states. The dynamics transducer of $\rdp^\eta_1$ has two states because when
  the last observation state is $s_3$ it is the same as when the last
  observation state is $s_2$.
  Thus, in $\rdp^\eta_1$, action $a_2$ is optimal when the last observation
  state is $s_3$; it is instead $a_1$ in $\rdp^\eta_2$.
  This assessment is required in order to compute an $\epsilon$-optimal policy
  with $\epsilon < 1$ when the underlying RDP can be $\rdp^\eta_1$ or
  $\rdp^\eta_2$. Therefore, it is necessary to observe at least one transition
  when the last observation state is $s_3$.
  However, the probability to never see $s_3$ in $N$ steps is $(1-\eta)^N$.
  If $N = \operatorname{poly}(1/\epsilon,\ln (1/\delta),\vec{d}_\rdp')$,
  then $(1-\eta)^N$ can be made arbitrarily high by choosing an arbitrarily
  small value for $\eta$, since none of the parameters in $\vec{d}_\rdp'$
  increases when $\eta$ decreases.
\end{proof}

The following lemma shows that the length of episodes to consider can grow with
the number of transducer states $n$.
\begin{lemma}
  For every algorithm there is an RDP $\rdp$ such that the algorithm does not
  necessarily reach accuracy $\epsilon$ and confidence $\delta$ in a number of
  steps $\operatorname{poly}(1/\epsilon,\ln (1/\delta),\vec{d}_\rdp')$ if 
  $\vec{d}_\rdp'$ is $\vec{d}_\rdp$ after removing number of transducer states.
\end{lemma}
\begin{proof}
  Consider $i \in \{1,2\}$, $n \geq 1$, states $S = \{ s_1, \dots, s_n \}$, and
  actions $A = \{ a_1,a_2 \}$.
  We describe an RDP $\rdp^n_i$.
  The initial observation state is $s_1$.
  When the last observation state is $s_j$ with $j \in [1,n)$, all actions yield
  observation state $s_{j+1}$ with probability $1$. The reward in all the above
  transitions is zero.
  When the last observation state is $s_n$, action $a_i$ yields reward $1$ and 
  observation state $s_n$ with probability $1$, and the other actions yields
  reward zero and observation state $s_n$ with probability $1$.
  The minimum dynamics transducer of $\rdp^n_i$ has states 
  $Q = \{ q_0, q_1, \dots, q_n \}$, where $q_0$ is the initial state, and $q_j$
  for $j \in [1,n]$ is the state when the last observation is $s_j$.

  To determine the best action when the last observation state is $s_n$ and
  hence compute an $\epsilon$-optimal policy with $\epsilon < 1$, it is required
  to see at least one transition when the last observation state is $s_n$, to
  assess whether the underlying RDP is $\rdp^n_1$ or $\rdp^n_2$.
  This requires an episode of length $n$.
  The value of $n$ can be chosen such that 
  $n > \operatorname{poly}(1/\epsilon,\ln (1/\delta),\vec{d}_\rdp')$, since 
  none of the parameters in $\vec{d}_\rdp'$ increases with $n$.
\end{proof}

The following lemma shows that the required number of action steps can grow with
the reachability parameter $1/\rho$.
\begin{lemma}
  For every algorithm there is an RDP $\rdp$ such that the algorithm does not
  necessarily reach accuracy $\epsilon$ and confidence $\delta$ in a number of
  steps $\operatorname{poly}(1/\epsilon,\ln (1/\delta),\vec{d}_\rdp')$ if 
  $\vec{d}_\rdp'$ is $\vec{d}_\rdp$ after removing the reachability parameter.
\end{lemma}
\begin{proof}
  Consider $i \in \{1,2\}$, $n \geq 1$, states $S = \{ s_1, \dots, s_n,
  \mathit{ended} \}$, and
  actions $A = \{ a_1,a_2 \}$.
  We describe an RDP $\rdp^n_i$.
  The initial observation state is $s_1$.
  When the last observation state is $s_j$ with $j \in [1,n)$, all actions yield
  observation state $s_{j+1}$ with probability $1/2$, and observation state
  $\mathit{ended}$ with probability $1/2$. 
  When the last observation state is $\mathit{ended}$, all actions yield
  observation state $\mathit{ended}$ with probability $1$.
  The reward in all the above transitions is zero.
  When the last observation state is $s_n$, action $a_i$ yields reward $1$ and 
  observation state $s_n$ with probability $1$, and the other actions yields
  reward zero and observation state $s_n$ with probability $1$.
  The minimum dynamics transducer of $\rdp^n_i$ has states 
  $Q = \{ q_0, q_1, \dots, q_n, \mathit{sink} \}$, where $q_0$ is the initial state, and $q_j$
  for $j \in [1,n]$ is the state when the last observation is $s_j$.

  To determine the best action when the last observation state is $s_n$ and
  hence compute an $\epsilon$-optimal policy with $\epsilon < 1$, it is required
  to see at least one transition when the last observation state is $s_n$, to
  assess whether the underlying RDP is $\rdp^n_1$ or $\rdp^n_2$.
  The probability of seeing observation state $s_n$ in an episode is $(1/2)^n$.
  The probability of seeing observation state $s_n$ in $N$ episodes is at most 
  $N \cdot (1/2)^n$, by a union bound.
  If $N = \operatorname{poly}(1/\epsilon,\ln (1/\delta),\vec{d}_\rdp')$, 
  we can choose $n$ such that $N \cdot (1/2)^n$ is arbitrarily small.
\end{proof}

The remaining lemmas are based on a core lower bound for the \emph{multi-armed
bandit problem} given in \cite{mannor2004sample}.
\begin{proposition}[Mannor and Tsitsiklis, 2004; Theorem 1]
  \label{prop:mannor}
  There exist positive constants $c_1$, $c_2$, $\epsilon_0$, and $\delta_0$
  such that for every $K \geq 2$, $\epsilon \in (0,\epsilon_0)$, $\delta \in
  (0,\delta_0)$, and for every algorithm, there is a $K$-armed bandit problem
  for which the number of action steps for the algorithm to reach
  accuracy $\epsilon$ and confidence $\delta$ in the worst case is at least 
  $$c_1 \cdot \frac{K}{\epsilon^2} \cdot \log \frac{c_2}{\delta}.$$
  In particular, $\epsilon_0$ and $\delta_0$ can be taken equal to $1/8$ and
  $e^{-4}/4$, respectively.
  Furthermore, the mentioned $K$-armed bandit problem has rewards in $\{0,1\}$,
  and its maximum reward probability is $1/2 + \epsilon$.
\end{proposition}
Note that the bound was originally stated for the expected number of
action steps, which is a lower bound for the actual number of action steps in
the worst-case---this was also pointed out in \cite{strehl2009reinforcement}.
Note also that the original theorem was for algorithms that, as a last step,
output a policy and then stop.
However, the proof of the theorem shows that, when the number of action steps is
lower than the claimed bound, the probability of selecting the wrong arm is
bigger than $\delta$, regardless of whether the algorithm stops or not.

The following lemma is implied by the previous proposition since the
\emph{number of actions} corresponds to the number of arms.
\begin{lemma} \label{lemma:lower-actions}
  For every algorithm there is an RDP $\rdp$ such that the algorithm does not
  necessarily reach accuracy $\epsilon$ and confidence $\delta$ in a number of
  steps $\operatorname{poly}(1/\epsilon,\ln (1/\delta),\vec{d}_\rdp')$ if 
  $\vec{d}'$ is $\vec{d}_\rdp$ after removing the number of actions.
\end{lemma}
\begin{proof}
  We show RDPs that encode MAB problems and argue that if this lemma
  were false then Proposition~\ref{prop:mannor} would be false.
  Consider a $K$-armed bandit problem with arm probabilities $p_1, \dots, p_K$,
  rewards in $\{0,1\}$, and maximum reward probability $1/2 + \epsilon$.
  The corresponding RDP has observation states 
  $S = \{ s_0, s^+, s^- \}$ and actions $A = \{ a_1, \dots, a_K \}$.
  The initial observation state is $s_0$.
  Regardless of the last observation state, when action $a_i$ is
  performed, the agent observes 
  $s^+$ with probability $p_i$ receiving reward $1$ and, it observes $s^-$ 
  with probability $1-p_i$ and receiving reward zero.
  Note that there is one set of parameters $\vec{d}_\rdp'$ that describes
  all such RDPs.
  In particular, the discount factor is irrelevant and can be
  taken equal to zero, the maximum reward is one, the reachability is one and
  the distinguishability is one because the dynamics function is state-less, 
  and the degree of determinism is $1/2 - \epsilon$ because it is the minimum
  probability that an arm issues a reward.

  Assume by contradiction that, in each of the above RDPs, an agent can reach
  accuracy $\epsilon \cdot r$ and confidence $\delta$ in
  $\operatorname{poly}(1/\epsilon,\ln 1/\delta,\vec{d}_\rdp')$ action
  steps. It amounts to solving the corresponding MAB problem with accuracy
  $\epsilon$ and confidence $\delta$ in  
  $\operatorname{poly}(1/\epsilon,\ln 1/\delta,\vec{d}_\rdp')$ action
  steps.
  For any given $\epsilon$ and $\delta$, the former quantity is constant, and
  hence will be smaller than the bound of Proposition~\ref{prop:mannor} for
  a sufficiently big number of arms $K$. This contradicts 
  Proposition~\ref{prop:mannor}.
\end{proof}

The following lemma is based on the observation that, if the number of action
steps did not depend on the \emph{maximum reward} value $\rmax$, then we could
artificially increase all the rewards given to the agent by an arbitrary factor
in order to obtain a higher accuracy in the same number of steps.
\begin{lemma} \label{lemma:lower-reward}
  For every algorithm there is an RDP $\rdp$ such that the algorithm does not
  necessarily reach accuracy $\epsilon$ and confidence $\delta$ in a number of
  steps $\operatorname{poly}(1/\epsilon,\ln (1/\delta),\vec{d}_\rdp')$ if 
  $\vec{d}'$ is $\vec{d}_\rdp$ after removing the maximum reward.
\end{lemma}
\begin{proof}
  Consider the encoding of the MAB problem introduced in the proof of
  Lemma~\ref{lemma:lower-actions}, with the difference that, each arms yields
  reward $\rmax$ instead of one.
  Note that the number of actions $|A|$ is the only parameter that can change in 
  $\vec{d}_\rdp'$ for different RDPs.

  Assume by contradiction that, in each of the above RDPs, an agent can reach
  accuracy $\epsilon \cdot \rmax$ and confidence $\delta$ in
  $\operatorname{poly}(1/(\epsilon \cdot \rmax),\ln
  (1/\delta),\vec{d}_\rdp')$ steps. 
  We can use this agent to solve the original MAB problem---the one with rewards
  in $\{ 0,1\}$---with accuracy $\epsilon$ and confidence
  $\delta$ in $N = \operatorname{poly}(1/(\epsilon \cdot \rmax),\ln
  (1/\delta),\vec{d}_\rdp')$ action steps.
  It suffices to give to the agent reward $r \cdot \rmax$ whenever the original
  MAB issues reward $r$---i.e., to give $\rmax$ when reward one is issued by the
  original MAB, and zero otherwise.
  Then, if the computed $(\epsilon \cdot \rmax)$-optimal policy selects action
  $a_i$, the $i$-th arm in the original MAB is $\epsilon$-optimal, which solves
  the original MAB.
  This contradicts Proposition~\ref{prop:mannor} because the number $N$ of
  actions steps can be made arbitrarily small by increasing the value of
  $\rmax$.
\end{proof}

The proof of the following lemma is a variation of the previous one. 
It builds on the observation that, in the setting with discounted rewards, even
if $\rmax = 1$, there can be a state whose value is arbitrarily large, for
decreasing values of the \emph{discount factor}.
\begin{lemma}
  For every algorithm there is an MDP $\rdp$ such that the algorithm does not
  necessarily reach accuracy $\epsilon$ and confidence $\delta$ in a number of
  steps $\operatorname{poly}(1/\epsilon,\ln (1/\delta),\vec{d}_\rdp')$ if 
  $\vec{d}_\rdp'$ is $\vec{d}_\rdp$ after removing the discount parameter.
\end{lemma}
\begin{proof}
  The proof is based on Proposition~\ref{prop:mannor} similarly to the proof of
  Lemma~\ref{lemma:lower-reward}. However, it considers a different encoding of
  the MAB problem, to ensure that there is an action whose value increases with
  the discount parameter---instead of the maximum reward parameter.

  Consider a $K$-armed bandit problem with arm probabilities $p_1, \dots, p_K$,
  rewards in $\{0,1\}$, and maximum reward probability $1/2 + \epsilon$.
  The corresponding RDP has observation states 
  $S = \{ s_0, s^+, s^- \}$ and actions 
  $A = \{ a_1, \dots, a_K \}$.
  The initial observation state is $s_0$.
  When the last observation state is $s_0$, action $a_i$ yields observation
  state $s^+$ with probability $p_i$ and observation
  state $s^-$ with probability $1-p_i$, and the reward is zero in both cases.
  When the last observation state is $s^+$, all actions yield $s^+$ with reward
  one.
  When the last observation state is $s^-$, all actions yield $s^-$ with reward
  zero.
  Clearly, during exploration, the agent can always go back to state $s_0$ by
  performing a stop action---as assumed throughout the paper.

  The key observation is that the value of every action in state $s_0$ is given
  by its probability of
  leading to state $s^+$ times the value
  $g = \sum_{i=2}^\infty \gamma^i$, which is finite and it amounts to 
  $\gamma^2/(1-\gamma)$. In particular, $g$ can be made arbitrarily big by
  making $\gamma$ arbitrarily close to $1$.
  Thus, $g$ plays the same role as $\rmax$ in Lemma~\ref{lemma:lower-reward},
  and we omit the rest of the proof since it proceeds similarly to the one of
  Lemma~\ref{lemma:lower-reward}.
\end{proof}

\subsection{Proof of Lemma~\ref{lemma:stationary-optimal-policy}}

The goal of this section is to prove Lemma~\ref{lemma:stationary-optimal-policy}.
Let the target RDP $\rdp$ and the ideal MDP $M$ be defined as in
Section~\ref{sec:algorithm}.
First, the optimal value functions $\valfunc_*$ and $\avalfunc_*$ of
$\rdp$ can be expressed in terms of the optimal value functions $\valfunc_*^M$
and $\avalfunc_*^M$ of $M$.
\begin{lemma} \label{lemma:value-function-rewrite}
  The value functions $\valfunc_*$ and $\avalfunc_*$ of $\rdp$ can be expressed
  as
  $\valfunc_*(h) = \valfunc_*^M(\tau(h))$ and
  $\avalfunc_*(h,a) = \avalfunc_*^M(\tau(h),a)$.
\end{lemma}
\begin{proof}
  We first show the claim for the value function $\valfunc_*$.
  We rewrite $\valfunc_*$ by expressing the dynamics $D$ of $\rdp$ in terms
  of the output function of the dynamics transducer $T$.
  \begin{align*}
    \valfunc_*(h) & =\max_a \sum_{sr} \dynfunc(s,r|h,a) \cdot (r + \gamma
    \cdot \valfunc_*(hs))
    \\
    \valfunc_*(h) & =\max_a \sum_{sr} \theta(\tau(h))(a,s,r) \cdot (r +
    \gamma
    \cdot \valfunc_*(hs))
  \end{align*}
  The equation above can be rewritten in terms of a new function
  $\bar{v}_*$ over states $Q$.
  In particular, $\valfunc_*(h) = \bar{\valfunc}_*(\tau(h))$.
  \begin{align*}
    \bar{\valfunc}_*(\tau(h)) & = \max_a  \sum_{sr} \theta(\tau(h))(a,s,r) \cdot
    (r + \gamma \cdot \bar{\valfunc}_*(\tau(\tau(h),s)))
  \end{align*}
  The definition of $\bar{v}_*$ can be given directly over $Q$.
  \begin{align*}
    \bar{\valfunc}_*(q) & = \max_a  \sum_{sr} \theta(q)(a,s,r) \cdot (r
    + \gamma \cdot \bar{\valfunc}_*(\tau(q,s)))
  \end{align*}
  We apply the definition of $\dynfunc^M$.
  \begin{align*}
    \bar{\valfunc}_*(q) & =  \max_a \sum_{q'r}  \dynfunc^M(q',r|q,a) \cdot (r
    + \gamma \cdot \bar{\valfunc}_*(q'))
  \end{align*}
  From the last equation, we can see that
  $\bar{\valfunc}_*(q) = \valfunc_*^M(q)$, and hence $\valfunc_*(h) =
  \valfunc^M_*(\tau(h))$ as claimed.

  For the action-value function $\avalfunc_*$, the steps are similar.
  One difference is that we make use of the result above for $\valfunc_*$.
  \begin{align*}
    \avalfunc_*(h,a) & = \sum_{sr} \dynfunc(s,r|h,a) \cdot (r + \gamma
    \cdot \valfunc_*(hs))
    \\
    \avalfunc_*(h,a) & = \sum_{sr} \theta(\tau(h))(a,s,r) \cdot (r + \gamma
    \cdot \valfunc_*(hs))
    \\
    \bar{\avalfunc}_*(q,a) & = \sum_{sr} \theta(q)(a,s,r) \cdot (r + \gamma
    \cdot \valfunc^M_*(\tau(q,s)))
    \\
    \bar{\avalfunc}_*(q,a) & = \sum_{q'r} \dynfunc^M(q',r|q,a) \cdot (r + \gamma
    \cdot \valfunc^M_*(q'))
  \end{align*}
  From the last equation, we can see that
  $\bar{\avalfunc}_*(q,a) = \avalfunc_*^M(q,a)$, and hence 
  $\avalfunc_*(h,a) = \avalfunc^M_*(\tau(h),a)$ as claimed.
\end{proof}



Then we show that, for every policy $\pi$ for $\rdp$ that can be expressed in
terms of a stationary policy $\pi'$ for $M$ of a certain form, the value of
$\pi$ in $\rdp$ can be expressed as the value of $\pi'$ in $M$.
\begin{lemma} \label{lemma:value-function-rewrite-policy}.
  Let $\pi'$ be a policy for $M$, and let $\pi = \pi'(\tau(\cdot))$.
  Then, $\valfunc_\pi(\cdot) = \valfunc^M_{\pi'}(\tau(\cdot))$.
\end{lemma}
\begin{proof}
  The proof technique is the one from Lemma~\ref{lemma:value-function-rewrite}.
  \begin{align*}
    \valfunc_\pi(h) & =  \sum_{asr} \pi(a|h) \cdot \dynfunc(s,r|h,a) \cdot (r +
    \gamma \cdot \valfunc_\pi(hs))
    \\
    \valfunc_\pi(h) & = \sum_{asr} \pi(a|h) \cdot
    \theta(\tau(h))(a,s,r) \cdot (r + \gamma \cdot \valfunc_\pi(hs))
    \\
    \bar{\valfunc}_\pi(q) & =  \sum_{asr} \pi(a|q) \cdot \theta(q)(a,s,r) \cdot (r
    + \gamma \cdot \bar{\valfunc}_\pi(\tau(q,s)))
    \\
    \bar{\valfunc}_\pi(q) & =  \sum_{aq'r} \pi(a|h) \cdot \dynfunc^M(q',r|q,a) \cdot (r
    + \gamma \cdot \bar{\valfunc}_\pi(q'))
  \end{align*}
  From the last equation we can see that $\bar{\valfunc}(q) = \valfunc^M(q)$,
  and hence $\valfunc_*(h) = \valfunc^M_*(\tau(h))$ as claimed.
\end{proof}

The previous two lemmas imply the main lemma of this section.
\lemmastationaryoptimalpolicy*
\begin{proof}
  Let $\pi'(\cdot) = \pi(\tau(\cdot))$.
  Assume $\max_q |\valfunc^M_{\pi}(q) - \valfunc^M_*(q)| \leq \epsilon$.
  The former is equivalent to
  $\max_h |\valfunc^M_{\pi}(\tau(h)) - \valfunc^M_*(\tau(h))| \leq \epsilon$.
  The former can be rewritten as
  $\max_h |\valfunc_{\pi'}(h) - \valfunc^M_*(\tau(h))| \leq \epsilon$ 
  by Theorem~\ref{lemma:value-function-rewrite-policy}.
  The former can be rewritten as
  $\max_h |\valfunc_{\pi'}(h) - \valfunc_*(h)| \leq \epsilon$
  by Theorem~\ref{lemma:value-function-rewrite}.
\end{proof}

\subsection{Proof of Lemma~\ref{lemma:approximation-of-mdp}}

The goal of this section is to prove Lemma~\ref{lemma:approximation-of-mdp}, by
applying known results for MDPs
\cite{singh1994upper,strehl2005theoretical,strehl2009reinforcement} similarly to
what is done in \cite{strehl2009reinforcement}.
First, similar MDPs have similar action-value functions.
Note that $\| \cdot \|_1$ is the $L_1$-norm.
\begin{proposition}[Strehl and Littman, 2005; Lemma 4]
  \label{prop:strehl-littman-2005}
  Let $M_1 = \langle A,S,R, \transitionfunc_1, \rewardfunc_1, \gamma \rangle$
  and $M_2 = \langle A,S,R, \transitionfunc_2, \rewardfunc_2, \gamma \rangle$ 
  be two MDPs. 
  Let $\rmax$ be the maximum value in $R$,
  let $\erewardfunc_i(s,a)$ be the expected reward after doing action $a$ in
  state $s$ in the MDP $M_i$, and
  let $\avalfunc^1_\pi$ and $\avalfunc^2_\pi$ be the action-value functions of
  $M_1$ and $M_2$ for a policy $\pi$.
  If $|\erewardfunc_1(s,a) - \erewardfunc_2(s,a)| \leq \alpha$ and 
  $\| \transitionfunc_1(\cdot|s,a) - \transitionfunc_2(\cdot|s,a) \|_1 \leq
  2 \cdot \beta$ for every state $s$ and action $a$, then the following
  condition holds for every state $s$ and action $a$:
  \begin{equation*}
    |\avalfunc^1_\pi(s,a) - \avalfunc^2_\pi(s,a)| \leq \frac{(1-\gamma) \cdot
    \alpha + \gamma \cdot \beta \cdot
    \rmax}{(1-\gamma) \cdot (1-\gamma+\beta \cdot \gamma)}.
  \end{equation*}
\end{proposition}

In our case we have a bound on the accuracy of the dynamics function, which
transfers as follows.
\begin{proposition}
  \label{prop:strehl-littman-2005-1}
  Let $M_1$, $M_2$, $\erewardfunc_1$, $\erewardfunc_2$, and $\rmax$ be as in
  Proposition~\ref{prop:strehl-littman-2005}, and let $\dynfunc_i$ be the
  dynamics function of $M_i$.
  The following holds, for every state $s$ and action $a$:
  $$|\erewardfunc_1(s,a) - \erewardfunc_2(s,a)| \leq \rmax \cdot \| \dynfunc_1(\cdot|s,a) -
  \dynfunc_2(\cdot|s,a) \|_1.$$
\end{proposition}
\begin{proof}
  We have $\erewardfunc_i(s_1,a) = \sum_r r \cdot \sum_{s_2} \dynfunc_i(s_2,r|s_1,a)$.
  Then, 
  \begin{gather*}
  |\erewardfunc_1(s,a) - \erewardfunc_2(s,a)| = 
  \\ 
  \left|\sum_r r \cdot \sum_{s_2} \dynfunc_1(s_2,r|s_1,a)-
  \dynfunc_2(s_2,r|s_1,a) \right| \leq
  \\ 
  \sum_r r \cdot \sum_{s_2} \left| \dynfunc_1(s_2,r|s_1,a)-
  \dynfunc_2(s_2,r|s_1,a) \right| \leq
  \\ 
  \rmax \cdot \sum_r \cdot \sum_{s_2} \left| \dynfunc_1(s_2,r|s_1,a)-
  \dynfunc_2(s_2,r|s_1,a) \right| =
  \\
  \rmax \cdot \| \dynfunc_1(\cdot|s,a) - \dynfunc_2(\cdot|s,a) \|_1.
  \end{gather*}
  This concludes the proof.
\end{proof}

\begin{proposition}
  \label{prop:strehl-littman-2005-2}
  Let $M_1$ and $M_2$ be as in Proposition~\ref{prop:strehl-littman-2005}, and
  let $\dynfunc_i$ be the dynamics function of $M_i$.
  The following holds, for every state $s$ and action $a$:
  $$\|\transitionfunc_1(\cdot|s,a) - \transitionfunc_2(\cdot|s,a)\|_1 \leq \|
  \dynfunc_1(\cdot|s,a) - \dynfunc_2(\cdot|s,a) \|_1.$$
\end{proposition}
\begin{proof}
  We have $\dynfunc_i(s_2|s_1,a) = \sum_r \dynfunc_i(s_2,r|s_1,a)$.
  Then, for every state $s_1$ and action $a$,
  \begin{gather*}
  \|\transitionfunc_1(\cdot|s_1,a) - \transitionfunc_2(\cdot|s_1,a)\|_1 =
  \\ 
  \sum_{s_2} \left|\transitionfunc_1(s_2|s_1,a) - \transitionfunc_2(s_2|s_1,a)
  \right| =
  \\ 
  \sum_{s_2} \left| \sum_r \dynfunc_1(s_2,r|s_1,a) -
  \dynfunc_2(s_2,r|s_1,a)\right| \leq
  \\ 
  \sum_{s_2}  \sum_r \left| \dynfunc_1(s_2,r|s_1,a) - \dynfunc_2(s_2,r|s_1,a)
  \right| =
  \\ 
  \| \dynfunc_1(s_2,r|s_1,a) - \dynfunc_2(s_2,r|s_1,a) \|_1.
  \end{gather*}
  This concludes the proof.
\end{proof}

We apply the previous proposition to bound the action-value function
of an MDP $\hat{M}$ that is an approximation of $M$.
In the following, to simplify the notation, we identify the states of
$\hat{M}$ with the ones of $M$; this is w.l.o.g.\ since we could rename each
state $q$ of $\hat{M}$ with $\phi(q)$ where $\phi$ is the bijection mentioned in
Definition~\ref{def:mdp-approximation}.
\begin{lemma} \label{lemma:bound-action-value}
  If $\hat{M}$ is an $\alpha$-approximation of $M$, then
  the following holds for every state $q$, action $a$, and stationary policy
  $\pi$:
  \begin{equation*}
    |\avalfunc^{\hat{M}}_\pi(q,a) - \avalfunc^M_\pi(q,a)| \leq 
    \frac{\alpha \cdot \rmax }{(1-\gamma)^2}.
  \end{equation*}
\end{lemma}
\begin{proof}
  Let $\rmax$ be the maximum value in $R$, and
  let $\erewardfunc(s,a)$ and $\hat{\erewardfunc}(s,a)$ be the expected
  reward after doing action $a$ in state $s$ in the MDP $M$ and $\hat{M}$,
  respectively.
  Let $\alpha'$ and $\beta$ be such that
  $|\erewardfunc(s,a) - \hat{\erewardfunc}(s,a)| \leq \alpha'$ and 
  $\| \transitionfunc(\cdot|s,a) - \hat{\transitionfunc}(\cdot|s,a) \|_1 \leq
  2 \beta$ for every state $s$ and action $a$.
  Then,
  \begin{gather*}
    |\avalfunc^{\hat{M}}_\pi(q,a) - \avalfunc^M_\pi(q,a)| \leq
    \frac{(1-\gamma) \cdot \alpha + \gamma \cdot \beta \cdot
    \rmax}{(1-\gamma) \cdot (1-\gamma+\beta \cdot \gamma)}
    \leq
    \frac{(1-\gamma) \cdot \rmax \cdot \epsilon + \gamma \cdot \epsilon \cdot
    \rmax}{(1-\gamma) \cdot (1-\gamma+\epsilon \cdot \gamma)}
    =
    \\
    \frac{\alpha \cdot \rmax }{(1-\gamma) \cdot (1-\gamma+\alpha \cdot \gamma)}
    \leq
    \frac{\alpha \cdot \rmax }{(1-\gamma)^2}.
  \end{gather*}
  The first inequality holds by
  Proposition~\ref{prop:strehl-littman-2005},
  the second one by Propositions~\ref{prop:strehl-littman-2005-1}
  and~\ref{prop:strehl-littman-2005-2}, and
  the last one because it is obtained by removing a positive additive term at
  the denominator.
\end{proof}

A near-optimal policy for $M$ can be computed via the value iteration algorithm.
Specifically, value iteration allows us to compute a close approximation of the
action-value function of $\hat{M}$, and then the corresponding greedy
policy---that picks actions with maximum value---is near-optimal.
The proposition is a straighforward generalisation of Proposition~4 from 
\cite{strehl2009reinforcement} to the case where $\rmax$ is not necessarily $1$.
We rewrite their proof to include $\rmax$.
\begin{proposition}[Generalisation of Proposition~4 
  of {[}Strehl et al., 2009{]}]
  \label{prop:strehl2009}
  Consider an MDP $M'$, let $\gamma$ be its discount factor, let $\rmax$ be its
  maximum reward value, and
  let $\avalfunc^{M'}_*$ be its optimal action-value function.
  Let $\alpha > 0$ be any real number satisfying $\alpha < \rmax/(1-\gamma)$. 
  Suppose that value iteration is run on $M'$ for
  $\left\lceil \frac{1}{1-\gamma} \cdot \ln \frac{\rmax}{\alpha \cdot
    (1-\gamma)} \right\rceil$
  iterations where the action-value estimates are initialized
  to some value between $0$ and $\rmax/(1-\gamma)$. Then, the resulting action value
  estimates $\hat{\avalfunc}^{M'}_*$ satisfy 
  $|\hat{\avalfunc}^{M'}_*(s,a)-\avalfunc^{M'}_*(s,a)| \leq \alpha$ for
  every state $s$ and action $a$.
\end{proposition}
\begin{proof}
  Let $\hat{\avalfunc}^{M'}_0$ denote the initial action-value estimates, and
  let $\hat{\avalfunc}^{M'}_i$ denote the action-value estimates after the $i$-th iteration
  of value iteration.
  Then, let $\Delta_i = \max_{s,a} |\avalfunc^{M'}_*(s,a) -
  \hat{\avalfunc}^{M'}_i(s,a)|$.
  We have that
  \begin{align*}
    \Delta_i & = \max_{s,a} \left|
    \sum_{s'r} \dynfunc^{M'}(s',r|s,a) \cdot (r + \gamma \cdot
    \valfunc^{M'}_*(s')) -
    \sum_{s'r} \dynfunc^{M'}(s',r|s,a) \cdot 
    (r + \gamma \cdot \hat{\valfunc}^{M'}_{i-1}(s')) \right|
    \\
    & = \gamma \cdot \max_{s,a} \left|
    \sum_{s'r} \dynfunc^{M'}(s',r|s,a) \cdot 
    (\valfunc^{M'}_*(s') - \hat{\valfunc}^{M'}_{i-1}(s')) \right|
    \\
    & \leq \gamma \cdot \Delta_{i-1}.
  \end{align*}
  Using this bound along with the fact that $\Delta_0 \leq \rmax/ (1-\gamma)$
  shows that $\Delta_i \leq (\rmax \cdot \gamma^i) / (1-\gamma)$.
  Thus, in order to have $\Delta_i \leq \alpha$, it suffices that
  $i \geq \ln((\alpha \cdot (1-\gamma))/\rmax) / \ln \gamma$.
  Then, the theorem follows since
  \begin{equation*}
    \frac{\ln\left(\frac{\alpha \cdot (1-\gamma)}{\rmax}\right)}{\ln \gamma} =
    \frac{\ln\left(\frac{\rmax}{\alpha \cdot (1-\gamma)}\right)}{-\ln \gamma} 
    \leq
    \frac{1}{1-\gamma} \cdot
    \ln\left(\frac{\rmax}{\alpha \cdot (1-\gamma)}\right).
  \end{equation*}
  To see that the former inequality holds, it suffices to note that
  the inequality $-\ln \gamma \geq 1-\gamma$ follows from the 
  inequality $e^x \geq 1+x$.
\end{proof}

\begin{proposition}[Singh and Yee, 1994; Corollary 2]
  \label{prop:singh-yee}
  Consider an MDP $M'$, let $\gamma$ be its discount factor, and let
  $\avalfunc^{M'}_*$ be its optimal action-value function.
  Furthermore, let 
  $\hat{\avalfunc}_*$ be an estimate of $\avalfunc^{M'}_*$, and let
  let $\pi$ be the greedy policy with respect to $\hat{\avalfunc}_*$, i.e., 
  $\pi(s) = \argmax_a \hat{\avalfunc}_*(s,a)$.
  For any $\alpha > 0$, if 
  $|\hat{\avalfunc}_*(s,a) - \avalfunc^{M'}_*(s,a)| \leq \alpha$ for every
  state $s$ and action $a$, then 
  $|\valfunc^{M'}_\pi(s) - \valfunc^{M'}_*(s)| \leq 2 \alpha/(1-\gamma)$
  for every state $s$.
\end{proposition}

To prove the main result of this section, we apply
Lemma~\ref{lemma:bound-action-value} and the previous propositions above.
Note that $\epsm = \frac{(1-\gamma)^3 \cdot \epsilon}{3 \cdot
\rmax}$ as specified in Table~\ref{tab:quantities}.
\lemmaapproximationofmdp*
\begin{proof}
  Assume that $\hat{M}$ is an $\epsm$-approximation of $M$.
  By Lemma~\ref{lemma:bound-action-value},
  \begin{equation} \label{eq:lemmaapproximationofmdp-1}
    |\avalfunc^{\hat{M}}_\pi(q,a) - \avalfunc^M_\pi(q,a)| \leq 
    \frac{\epsm \cdot \rmax }{(1-\gamma)^2}.
  \end{equation}
  Consider to run value iteration on $\hat{M}$ for 
  $\left\lceil \frac{1}{1-\gamma} \cdot 
    \ln \left( \frac{2 \cdot \rmax}{\epsilon \cdot (1-\gamma)^2}
  \right)\right\rceil$ iterations, initialising all action-value estimates to
  zero.
  By Proposition~\ref{prop:strehl2009},
  it yields $\hat{\avalfunc}^{\hat{M}}_*$ such that
  $|\hat{\avalfunc}^{\hat{M}}_*(s,a)-\avalfunc^{\hat{M}}_*(s,a)| \leq
  1/2 \cdot \epsm \cdot \rmax / (1-\gamma)^2$. Note that there are two
  levels of approximation; namely, $\hat{\avalfunc}^{\hat{M}}_*$ is an
  approximation  of the action-value function
  $\avalfunc^{\hat{M}}_*$ of an approximation $\hat{M}$ of the MDP $M$.
  The former bound together with \eqref{eq:lemmaapproximationofmdp-1} yields
  the following:
  \begin{equation*}
    |\hat{\avalfunc}^{\hat{M}}_*(s,a)-\avalfunc^M_*(s,a)| \leq 
    \frac{1}{2} \cdot \frac{\epsm \cdot \rmax}{(1-\gamma)^2} + 
    \frac{\epsm \cdot \rmax }{(1-\gamma)^2}
    =
    \frac{3}{2} \cdot \frac{\epsm \cdot \rmax}{(1-\gamma)^2} =
    \frac{(1-\gamma) \cdot \epsilon}{2}.
  \end{equation*}
  Let $\pi$ be the greedy policy $\pi(q) = \argmax_a \hat{\avalfunc}_*(q,a)$.
  Thus, by Proposition~\ref{prop:singh-yee}, 
  $|\valfunc^M_\pi(q) - \valfunc^M_*(q)| \leq \epsilon$ for every state $q$, and
  hence $\pi$ is $\epsilon$-optimal for $M$.
\end{proof}

\subsection{Proof of Lemma~\ref{lemma:approximation-of-pdfa}}

To simplify the notation, we identify the states of
$\hat{\mathcal{A}}$ with the ones of $\mathcal{A}$; this is w.l.o.g.\ since we
could rename each state $q$ of $\hat{\mathcal{A}}$ with $\phi(q)$ where $\phi$
is the isomorphism mentioned in
Definition~\ref{def:pdfa-approximation}.
Note that $\epsa = \frac{(1-p) \cdot \epsm}{|A| \cdot n \cdot
|\Sigma|}$ as defined in Table~\ref{tab:quantities}.
\lemmaapproximationofpdfa*
\begin{proof}
  For any state $q$ and action $a$,
  \begin{gather*}
    \| \dynfunc^M(\cdot|q,a) - \dynfunc^{\hat{M}}(\cdot|q,a) \|_1 = 
    \\
    \sum_{q'}\sum_r | \dynfunc^M(q',r|q,a)
    - \dynfunc^{\hat{M}}(q',r|q,a)| =
    \\
    \sum_{q'}\sum_{sr:\tau(q,s) = q'} \left|
    \theta(q)(a,s,r) - \frac{|A|}{1-p} \cdot
    \hat{\lambda}(q,asr) \right| =
    \\
    \frac{|A|}{1-p} \cdot \sum_{q'} \sum_{sr:\tau(q,s) = q'} \left|
    \lambda(q,asr) - \hat{\lambda}(q,asr) \right| =
    \\
    \frac{|A|}{1-p} \cdot \left(\sum_{q'}\sum_{sr:\tau(q,s) = q' \land asr \in
    \Sigma}
    \left|
    \lambda(q,asr) - \hat{\lambda}(q,asr) \right| 
    +
    \sum_{q'} \sum_{sr:\tau(q,s) = q' \land asr \notin \Sigma} \left|
    \lambda(q,asr) - \hat{\lambda}(q,asr) \right| \right)
    \leq
    \\
    \frac{|A|}{1-p} \cdot \sum_{q'} \sum_{sr:\tau(q,s) = q' \land asr \in
    \Sigma} \epsa
    \leq
    \\
    \frac{|A|}{1-p} \cdot n \cdot |\Sigma| \cdot \epsa = \epsm.
  \end{gather*}
  The first equality holds by the definition of $L_1$-norm. 
  The second equality holds by the definition of $\dynfunc^M$ and
  $\dynfunc^{\hat{M}}$.
  The third equality holds by the definition of $\lambda$. 
  The fourth equality is simply a partitioning of the elements $q'sr$.
  The first inequality holds because $\hat{\mathcal{A}}$ is an
  $\epsa$-approximation of $\mathcal{A}$; in particular, for all elements
  $q'sr$ such that $asr \notin \Sigma$, $\lambda(q,asr) = 0$ and hence
  $\hat{\lambda}(q,asr) = 0$.
  The second inequality holds because there are at most $n$ elements $q'$ and at
  most $|\Sigma|$ elements $sr$.
\end{proof}

\subsection{Proof of Lemma~\ref{lemma:distinguishability}}

We split the lemma in three separate ones.
Note that $\mu$ is the distinguishability of $\rdp$, $n$ is the number of states
of the dynamics transducer $T$, and $p$ is the stop probability.

\par\smallskip\noindent\textbf{Lemma~\ref{lemma:distinguishability}.1}.
\textit{The distinguishability of $\mathcal{A}$ is at least $\mu \cdot (1-
p)^n$.}
\begin{proof}
  Consider two distinct states $q_1$ and $q_2$ of $\mathcal{A}$.
  It suffices to show a string $x \in \Sigma^+$ such that
  $|\lambda(q_1,x)-\lambda(q_2,x)| \geq \mu \cdot (1- p)^n$.
  Since $q_1$ and $q_2$ are distinct and $T$ is minimum, there are strings 
  $y_1, y_2, z \in \Sigma^*$ such that: 
  \begin{romanenumerate}
  \item
    $y_1z$ or $y_2z$ have non-zero probability of being prefixes of strings
    generated by
    $\mathcal{A}$, i.e., $\lambda(q_0,y_1z) > 0$ or
    $\lambda(q_0,y_2z) > 0$;
  \item
    $y_1$ and $y_2$ lead to $q_1$ and $q_2$, respectively, starting from the
    initial state $q_0$, i.e., $\tau'(q_0,y_1) = q_1$ and $\tau'(q_0,y_1) =
    q_1$;
  \item
    $z$ has non-zero probability of being a prefix of a string generated by
    $\mathcal{A}$ from state $q_1$ or $q_2$, i.e.,
    $|\lambda(q_1,z) - \lambda(q_2,z)| > 0$.
  \end{romanenumerate}
  Let $h_1$ and $h_2$ be the histories obtained from $y_1$ and $y_2$ by
  removing all action and reward symbols.
  By the construction of $\mathcal{A}$,
  we have $\tau(q_0,h_1) = q_1$ and $\tau(q_0,h_2) = q_2$,
  and hence
  $|\dynfunc_\upolicy(z|h_1) - \dynfunc_\upolicy(z|h_2)| > 0$ since the uniform
  policy $\upolicy$ does not decrease the probability of a trace $z$ with
  respect to the policy $\pi_p$---which selects actions uniformly at random like
  the uniform policy and, additionally, considers to stop at every step.
  Thus, $|\dynfunc_\upolicy(z|h_1) - \dynfunc_\upolicy(z|h_2)| > \mu$ because
  $\rdp$ is $\mu$-distinguishable.
  Since $\dynfunc$ is represented by a transducer with $n$ states, and
  $\upolicy$ is stateless, also $\dynfunc_\upolicy$ can be represented by a
  transducer with $n$ states. Thus, a pumping lemma argument implies that
  $|\dynfunc_\upolicy(z'|h_1) - \dynfunc_\upolicy(z'|h_2)| > \mu$ for some
  prefix $z'$ of $z$ having length at most $n$.
  We have that 
  $|\lambda(q_1,z')-\lambda(q_2,z')|$ is given by
  $|\dynfunc_\upolicy(z'|h_1) - \dynfunc_\upolicy(z'|h_2)|$ multiplied by the
  probability of not stopping for $|z'|$ times in a row, which is $(1- p)^n$.
  Therefore, 
  $|\lambda(q_1,z')-\lambda(q_2,z')| \geq \mu \cdot (1- p)^n$, which shows the
  lemma.
\end{proof}

\par\smallskip\noindent\textbf{Lemma~\ref{lemma:distinguishability}.2}.
\textit{The probability that a state of $\mathcal{A}$ is visited during a run is at
  least $\rho \cdot (1-p)^n$.}
\begin{proof}
  A state of $\mathcal{A}$ is visited if the corresponding state of $T$ is
  visited under policy $\pi_p$.
  The probability of visiting a state of $T$ under policy $\pi_p$ is at least
  the probability of visiting it in $n$ steps given that we perform $n$ steps,
  times the probability of performing $n$ steps.
  The former probability is at least $\rho$ by definition, and the latter
  probability is $(1-p)^n$.
\end{proof}

\par\smallskip\noindent\textbf{Lemma~\ref{lemma:distinguishability}.3}.
\textit{The minimum non-zero probability of a non-$\stopaction$ transition of
  $\mathcal{A}$ is at least $\eta \cdot (1-p)/|A|$.}
\begin{proof}
  The minimum non-zero probability of a non-$\stopaction$ transition of
  $\mathcal{A}$ can be established from the definition of $\mathcal{A}$. It is 
  the quantity $(1-p)/|A|$ multiplied by the minimum non-zero
  probability of an observation-state in $\rdp$ for a given history an action,
  wich is the degree of determinism $\eta$.
\end{proof}

\subsection{Proof of Lemma~\ref{lemma:number-of-episodes}}

\begin{table}
\centering
\begin{tabular}{ll}
  \toprule
  & Definition \\
  \midrule
  $\epsilon_0$ & $\frac{\epsilon_2' \cdot \epsilon_5}{\hat{n} \cdot |\Sigma|}$
  \\[5pt]
  $\delta_0$ & $\frac{\delta'}{\hat{n} \cdot (\hat{n} \cdot |\Sigma| + |\Sigma|
  + 1)}$
  \\[5pt]
  $N_0$         & $\frac{16 \cdot \euler}{\epsilon_0 \cdot (\mu')^2} \cdot
  \left( \ln\frac{16}{\epsilon_0 \cdot (\mu')^2} + \ln \frac{32 \cdot
    L}{\delta_0^2}\right)$       \\[7pt]
    $N_3$         & $\frac{1}{\epsilon_0 \cdot \epsilon_1^2} \cdot \ln
  \frac{2 \cdot (|\Sigma|+1)}{\delta_0}$       \\
  \bottomrule
\end{tabular}
\caption{Quantities for Proposition~\ref{prop:balle}. Note that $\euler$ is
Euler's number, and $L$ is the expected length of the strings generated by the
automaton.}
\label{tab:quantities-appendix}
\end{table}

The following proposition summarises the guarantees of the $\adact$ algorithm.
In particular, it summarises Definition~3 and Theorem~25 of
\cite{balle2013pdfa}.
Differently from the original statements, the following proposition takes
$\epsilon_1$, $\epsilon_2'$, and $\epsilon_5$ as primary
quantities.\footnote{We have named quantities in a way that highlights the
correspondence with \cite{balle2013pdfa}. Quantities $\epsilon_1$ and
$\epsilon_5$ are as in \cite{balle2013pdfa}. Quantity $\epsilon_2'$
corresponds to $\epsilon_2/(L+1)$, where $L$ is the expected length of strings. We
have maintained names also for $N_0$, $N_3$, and $\delta_0$.}
The reason is that \citeauthor{balle2013pdfa} were interested in learning likely
parts of the automaton, whereas we want to learn all the parts of the automaton,
and hence we impose upfront that they are sufficiently likely.
Note also that our notion of distinguishability yields a lower bound on their
notion of distinguishability for PDFA, and hence can be used in place of theirs. 
In particular, their distinguishability also takes into account the supremum
distance $\mathrm{L}_\infty$, which is not relevant to our setting---see
Definition~2 of \cite{balle2013pdfa}.
\begin{proposition}[Balle et al., 2013]
  \label{prop:balle}
  Let $\epsilon_1, \epsilon_2', \epsilon_5 > 0$, 
  let $\delta' \in (0,1)$, and 
  let $\mu' > 0$.
  Furthermore, 
  let $\hat{n}$ be a non-negative integer, and 
  let $X$ be a set of strings generated by $\mathcal{A}$.
  Consider the following conditions:
  \begin{enumerate}[label=(\alph*)]
    \item
      \label{item:th:balle-1}
      $\hat{n}$ is an upper bound on the number of states of $\mathcal{A}$,
    \item
      \label{item:th:balle-2}
      $\mathcal{A}$ is $\mu'$-distinguishable,
    \item
      \label{item:th:balle-3}
      the minimum probability that a given state of $\mathcal{A}$ is visited in
      a run is at least $\epsilon_2'$, 
    \item
      \label{item:th:balle-4}
      the minimum probability of a non-$\stopaction$ transition of $\mathcal{A}$
      is at least $\epsilon_5$,
    \item
      \label{item:th:balle-5}
      the size of $X$ is greater than $\max(N_0,N_3)$, where $N_0$ and $N_3$ are
      defined in Table~\ref{tab:quantities-appendix}.
  \end{enumerate}
  If conditions~\ref{item:th:balle-1}--\ref{item:th:balle-5} hold,
  then the $\adact$ algorithm on input $(\hat{n},|\Sigma|,\delta',X)$ 
  returns a PDFA $\hat{\mathcal{A}}$ that is an $\epsilon_1$-approximation of
  $\mathcal{A}$ with probability $1-\delta'$.
  Furthermore, the algorithm runs in time 
  $O(\hat{n}^2 \cdot |\Sigma| \cdot \|X\|)$, even if the above conditions
  do not hold.\footnote{$\|X\|$ denotes the size of $X$ defined as the sum (with
  repetitions) of the length of all its strings in $X$.}
\end{proposition}

We apply the previous proposition, using the bounds from
Lemma~\ref{lemma:distinguishability} to ensure
that conditions~\ref{item:th:balle-2}--\ref{item:th:balle-4} hold.
\lemmanumberofepisodes*
\begin{proof}
  It suffices to take $\epsilon_1 = \epsa$ and find values for 
  $\mu'$, $\epsilon_2'$, $\epsilon_5$, $\samplesone$, and $\samplestwo$ such that
  conditions \ref{item:th:balle-1}--\ref{item:th:balle-5} of
  Proposition~\ref{prop:balle} are satisfied.
  Since
  $p \leq 1/(10 \cdot n +1)$ implies $(1-p)^n \geq 0.9$,
  by Lemma~\ref{lemma:distinguishability}.1 we can take 
  $\mu' = 0.9 \cdot \mu$, and
  by Lemma~\ref{lemma:distinguishability}.2 we can take
  $\epsilon_2' = 0.9\cdot \rho$.
  Furthermore, by Lemma~\ref{lemma:distinguishability}.3
  we can take $\epsilon_5 = \eta \cdot (1-p)/|A|$.
  Finally, condition~\ref{item:th:balle-5} is satisfied since 
  $\samplesone$ and $\samplestwo$ are $N_0$ and $N_3$, respectively, when we
  instantiate the parameters specified above.
  Note also that $L = (1-p)/p$ since the length of episodes has a geometric
  distribution with parameter $p$.
\end{proof}

\subsection{Proof of Theorem~\ref{th:pac}}

\begin{table}
  \centering
  \begin{tabular}{rl}
    \toprule
    & Definition \\
    \midrule
    $\ell_1$         & $24.2 \cdot \euler \cdot \frac{|A| \cdot |\Sigma|}{\rho
                        \cdot \eta \cdot \mu^2}$   
    \\[5pt]
    $\ell_2$         & $10.1 + \ln \frac{|A| \cdot |\Sigma|^3}{\rho \cdot \eta \cdot \mu^2
    \cdot \delta^4}$   
    \\[5pt]
    $\ell_3$         & $11.31 \cdot \frac{|\Sigma|^3 \cdot |A|^3 \cdot n^2 \cdot
    \rmax^2}{\rho \cdot \eta \cdot (1-\gamma)^6 \cdot \epsilon^2}$   
    \\[5pt]
    $\ell_4$         & $\ln \frac{8 \cdot |\Sigma|}{\delta^2}$   
    \\
    \bottomrule
  \end{tabular}
  \caption{Quantities for Lemma~\ref{lemma:iterations}.}
  \label{tab:quantities-number-of-iterations}
\end{table}

Algorithm~\ref{alg:rl} computes correct estimates for the alphabet and the
maximum reward, when the number of collected episodes is sufficiently high.
\begin{lemma} \label{lemma:correct-rmax-alphabet}
  Consider the values $\hat{\Sigma}$ and $\hatrmax$ computed
  by Algorithm~\ref{alg:rl} at Line~11.
  If $|X| \geq \max(\samplesone,\samplestwo)$, then $\hatrmax = \rmax$ and
  $\hat{\Sigma} = \Sigma$.
\end{lemma}
\begin{proof}
  If $|X| \geq \max(\samplesone,\samplestwo)$, the $\adact$ algorithm correctly
  finds all transitions of $\mathcal{A}$, by Proposition~\ref{prop:balle}. For
  each such transition, the algorithm requires
  to see at least an example in $X$, and hence $\hatrmax = \rmax$ and 
  $\hat{\Sigma} = \Sigma$.
\end{proof}

Algorithm~\ref{alg:rl} computes an accurate automaton in polynomially-many
iterations.
\begin{lemma}
  \label{lemma:iterations}
  Let $\ell \geq \max(n,\sqrt[6]{8 \ln
  \delta},\ell_1,\ell_2,\ell_3,\ell_4)$ with $\ell_1$, $\ell_2$,
  $\ell_3$, and $\ell_4$ defined as in
  Table~\ref{tab:quantities-number-of-iterations}.
  Consider the automaton $\hat{\mathcal{A}}$ computed in Line~12 of
  Algorithm~\ref{alg:rl} at the $\ell$-th iteration of the main loop.
  Then, $\hat{\mathcal{A}}$ is an $\epsa$-approximation of $\mathcal{A}$
  with probability at least $1-\delta$.
\end{lemma}
\begin{proof}
  Consider the $\ell$-th iteration.

  Since $\ell \geq n$, we have that $p \leq 1/(10n+1)$ and $\adact$ is
  instantiated with $\ell \geq n$ as a correct upper bound for
  the number of states.
  Thus, the first two conditions of Lemma~\ref{lemma:number-of-episodes}
  are satisfied, and it remains to show that $|X| \geq
  \max(\samplesone,\samplestwo)$ with
  sufficient probability; note that probability $1-\delta^2$ suffices, since
  when it is multiplied by the probability of success $1-\delta^2$ of $\adact$,
  it is still above the required confidence $1-\delta$.
  The size of $X$ mentioned above also guarantees that 
  $\hat{\Sigma} = \Sigma$ (Line~11),
  and hence that $\adact$ is called with a correct value for the alphabet size.

  We first argue that $\ell \geq \sqrt[6]{8\ln \delta}$ ensures that the
  number of generated episodes $|X|$ is at least 
  $\ell^2 \cdot (\ell + 5 \ln \ell)$ with probability at least $1-\delta^2$.
  The algorithm performs 
  $k = 2 \cdot (10\ell + 1) \cdot \ell^2 \cdot (\ell + 5 \ln \ell)$ actions.
  The number of episodes corresponds to the number of stop actions among $k$
  actions.
  Thus, by a Chernoff bound,
  the probability that the $\ell$-th iteration generates less than 
  $\ell^2 \cdot (\ell + 5 \ln \ell)$ episodes is at most 
  $\exp(-\ell^2 \cdot (\ell + 5 \ln \ell)/4)$, which is at most $\delta^2$ since
  $\ell \geq \sqrt[6]{8\ln \delta}$.

  Then, it suffices to show 
  $\ell^2 \cdot (\ell + 5 \ln \ell) \geq \max(\samplesone,\samplestwo)$.
  The critical aspect is that also the values of $\samplesone$ and
  $\samplestwo$ depend on $\ell$, because $\hat{n}$ and $p$ in the definition of
  $\samplesone$ and $\samplestwo$ are two be instantiated as
  $\hat{n} = \ell$ and $p = 1/(10\ell+1)$.
  It is easy to verify that 
  $\ell^2 \cdot (\ell + 5 \ln \ell) \geq \max(\samplesone,\samplestwo)$ when
  $\ell \geq \max(n,\sqrt[6]{8 \ln
  \delta},\ell_1,\ell_2,\ell_3,\ell_4)$.
\end{proof}

The size of alphabet $\Sigma$ is bounded by the parameters in $\vec{d}_\rdp$.
Note that
$\Sigma = \{ asr \in ASR \mid \exists h.\ \dynfunc(s,r|h,a)>0 \}$ as defined in
Section~\ref{sec:algorithm}.
\begin{lemma} \label{lemma:alphabet-size}
  $|\Sigma| \leq |A| \cdot n \cdot \lceil 1/\eta \rceil$.
\end{lemma}
\begin{proof}
  For every history $h$ and action $a$, if 
  $\transitionfunc(\cdot|h,a)$ assigns non-zero probability to an observation
  state, then it assigns at least probability $\eta$ by definition of $\eta$;
  thus, $\transitionfunc(\cdot|h,a)$ assigns non-zero probability to at most
  $\lceil 1/\eta \rceil$.
  We have that $\dynfunc(s,r|h,a)$ is $\transitionfunc(s|h,a)$ if 
  $r = \rewardfunc(h,a,s)$ and zero otherwise, by definition; hence, 
  there are at most $\lceil 1/\eta \rceil$ pairs $s,r$ that
  are assigned non-zero probability by $\dynfunc(\cdot|h,a)$.
  Furthermore, the dynamics function can be expressed in terms of the transducer
  output, i.e., $\dynfunc(s,r|h,a) = \theta(\tau(q_0,h))(a,s,r)$, and hence the
  number of pairs $s,r$ that are assigned non-zero probability when the history
  is $h$ and the action is $a$ are the pairs that are assigned  
  non-zero probability when the transducer state $\tau(q_0,h)$ and the action is
  $a$.
  Since there are $n$ states and $|A|$ actions, 
  we have $|\Sigma| \leq |A| \cdot n \cdot \lceil 1/\eta \rceil$ as required.
\end{proof}

The main theorem of this section follows from
Lemmas~\ref{lemma:stationary-optimal-policy}--\ref{lemma:number-of-episodes}, 
\ref{lemma:iterations}, and \ref{lemma:alphabet-size}.
\thpac*
\begin{proof}
  By Lemma~\ref{lemma:stationary-optimal-policy},
  Algorithm~\ref{alg:rl} returns a policy that is $\epsilon$-optimal with
  confidence at least $1-\delta$ when the policy $\pi$ computed in Line~15 is
  $\epsilon$-optimal with confidence at least $1-\delta$.
  By Lemma~\ref{lemma:approximation-of-mdp}, this happens when
  $\hat{M}$ is an $\epsm$-approximation of $M$, assuming that 
  $\hatrmax = \rmax$.
  By Lemma~\ref{lemma:approximation-of-pdfa},
  $\hat{M}$ is an $\epsm$-approximation of $M$ if
  $\hat{\mathcal{A}}$ is an $\epsa$-approximation of $\mathcal{A}$.
  By Lemma~\ref{lemma:iterations},
  the former condition is true in the $\ell$-th iteration with
  $\ell \geq \max(n,\sqrt[6]{8 \ln
  \delta},\ell_1,\ell_2,\ell_3,\ell_4)$; this condition also guarantees that
  $\hatrmax = \rmax$, by Lemma~\ref{lemma:correct-rmax-alphabet}.
  We have that $\ell$ is polynomial in $\vec{d}_\rdp$ as given in
  \eqref{eq:parameters}, considering also the
  bound for $|\Sigma|$ given in Lemma~\ref{lemma:alphabet-size}.
  Since each iteration performs $O(\ell^4 \ln \ell)$ action steps (see
  Line~2), and the $\adact$ and $\valiter$ algorithms run in polynomial time, 
  we have that each iteration performs a polynomial number of steps.
  Therefore, the end of the $\ell$-th iteration is also reached in a 
  a polynomial number of steps, at which point the algorithm returns an
  $\epsilon$-optimal policy with probability at least $1-\delta$.
  Finally, since the above reasoning applies to the following iterations as
  well, every policy returned at a later moment is also 
  $\epsilon$-optimal with probability at least $1-\delta$.
  Therefore, Algorithm~\ref{alg:rl} is PAC-RL.
\end{proof}

\subsection{Proof of Theorem~\ref{th:additional-alg}}

\begin{table}
  \centering
  \begin{tabular}{ll}
    \toprule
    & Definition \\
    \midrule
    $\delta_0$ & $\frac{\delta}{\hat{n} \cdot (\hat{n} \cdot |\Sigma| +
    |\Sigma| + 1)}$
    \\[7pt]
    $N_0'$ & $24.2 \cdot \euler \cdot \frac{|A| \cdot |\Sigma| \cdot
    \hat{n}}{\rho \cdot \eta \cdot \mu^2} \cdot \left( 8.96 + \ln \frac{|A|
    \cdot \hat{n}^2 \cdot |\Sigma|}{\rho \cdot \eta \cdot \mu^2 \cdot
  \delta_0^2} \right)$
    \\[7pt]
    $N_3'$ & $13.31 \cdot \frac{|\Sigma|^3 \cdot |A|^3 \cdot \hat{n}^3 \cdot
    \rmax^2}{\rho \cdot \eta \cdot (1-\gamma)^6 \cdot \epsilon^2} \cdot \ln
    \frac{2 \cdot (|\Sigma|+1)}{\delta_0}$
    \\
    \bottomrule
  \end{tabular}
  \caption{Quantities for Theorem~\ref{th:additional-alg}.}
  \label{tab:quantities-additional}
\end{table}

\thadditionalalg*
\begin{proof}
  It suffices to show that the expected number of action steps is at most
  $10\hat{n} \cdot \max(N_0',N_3')$, where 
  $N_0'$ and $N_3'$ are given in Table~\ref{tab:quantities-additional}. 
  By the same arguments used for Theorem~\ref{th:pac}, based
  on Lemmas 
  \ref{lemma:stationary-optimal-policy}, 
  \ref{lemma:approximation-of-mdp},
  \ref{lemma:approximation-of-pdfa}, and
  \ref{lemma:number-of-episodes},
  the algorithm outputs an $\epsilon$-optimal policy with confidence
  $1-\delta$ if it has read $\max(N_0',N_3')$ episodes.
  In particular, the values $N_0'$ and $N_3'$ can be derived from the values
  $\samplesone$ and $\samplestwo$ of Table~\ref{tab:quantities}, respectively.
  The expected length of an episode is $10\hat{n}$, since the length of an
  episode is a geometric random variable with parameter $1/(10\hat{n}+1)$.
  Therefore, the expected number of action steps is 
  $10\hat{n} \cdot \max(N_0',N_3')$.
\end{proof}

\section{Additional Material for The Running Example}
\label{sec:running-example}

\paragraph{Formal description and transducer.}
In Example~\ref{ex:running-example-1} we have described a family of RDPs 
$\mathcal{P}_m = \langle A, S, R, \dynfunc, \gamma \rangle$ 
where 
$A = \{ a_1, a_2 \}$, 
$S = \{0,1\} \times [0,m-1] \times \{ \mathit{enemy},
\mathit{clear}\}$, $R = \{0,1\}$, $\gamma = 0$ (the discount
  factor is irrelevant in this case), and the dynamics function $\dynfunc$
is represented by the transducer
$\langle Q, q_0, S, \tau, \theta \rangle$ where:
\begin{itemize}
  \item 
    the states are
    $Q = [0,m-1] \times \{0,1\}$ where the first component denotes the current
    agent's column and the second component is a bit denoting which of the
    two sets of probabilities are being used by the enemies,
  \item
    the initial state is
    $q_0 = \langle m-1, 0 \rangle$, 
  \item
    the transition function is:
    \begin{itemize}
      \item
        $\tau(\langle i, b \rangle, \langle j, i+1 \operatorname{mod} m,
        \mathit{enemy} \rangle) =
        \langle i+1 \operatorname{mod} m, b+1 \operatorname{mod} 2 \rangle$,
      \item
        $\tau(\langle i, b \rangle, \langle j, i+1\operatorname{mod} m,
        \mathit{clear} \rangle) =
        \langle i+1\operatorname{mod} m, b \rangle$,
    \end{itemize}
  \item
    the output function is:
    \begin{itemize}
      \item
        $\theta(\langle i, b \rangle)(a_k,\langle k, j, \mathit{enemy}
        \rangle, 0) = p_j^b$ where $j = i+1\operatorname{mod} m$,
      \item
        $\theta(\langle i, b \rangle)(a_k,\langle k, j, \mathit{clear}
        \rangle, 1) = 1-p_j^b$ where $j = i+1\operatorname{mod} m$.
    \end{itemize}
\end{itemize}

\paragraph{Value of the parameters.}
The reachability $\rho$ of the RDP $\rdp$ above is $1/2$.
It suffices to show, by induction, that every state reachable in $i$ steps has
probability at least $1/2$ to be visited at step $i$, under the uniform policy.
For the base case we have that the initial state $\langle m-1, 0 \rangle$ has
probability $1$ to be visited at step $0$.
By the inductive hypothesis, states $\langle i-1, 0 \rangle$ and 
$\langle i-1, 1 \rangle$ have probability at least $1/2$ to be visited at step
$i \geq 1$.
Thus, the probability of visiting 
$\langle i, 0 \rangle$ or $\langle i, 1 \rangle$ at step $i+1$ is at least
$$\frac{1}{2} \cdot \left(\frac{1}{|A|} \cdot p_{i-1}^0 + \frac{1}{|A|} \cdot
(1-p_{i-1}^0)\right)
+ \frac{1}{2} \cdot 
\left(\frac{1}{|A|} \cdot p_{i-1}^1 + \frac{1}{|A|}\cdot (1-p_{i-1}^1)\right) =
\frac{1}{2}.$$

The degree of determinism $\eta$ of $\rdp$ is the minimum value returned by
$\theta$, which is the minimum, for any $i$, among $p_i^0$, $p_i^1$,
$1-p_i^0$, and $1-p_i^1$.

The distinguishability $\mu$ of $\rdp$ is determined by the minimum probability
difference of an observation in two different states.
For every two states $\langle i, b \rangle$ and $\langle j, c \rangle$ with 
$i \neq j$, we have that there is an observation 
$\langle k, i+1, e \rangle$ that has probability at least $\eta$ in
the former state and probability zero in the latter state.
For every two states $\langle i, b \rangle$ and $\langle i, c \rangle$ with 
$b \neq c$, we have that the difference in the probability of 
$\langle k, i+1, \mathit{enemy} \rangle$ is at least $|p_i^0-p_i^1|$.
Taking into account the probability of choosing an action uniformly at random,
which is $1/2$, we have that the distinguishability is at least 
$1/2 \min(\eta,\min_i(|p_i^0-p_i^1|))$.

\section{Comparison with [Abadi and Brafman, 2020]}
\label{sec:comparison}

We compare our approach with the $\stm$ algorithm from
\cite{abadi2020learning} on the family of RDPs $\rdp_m$ introduced in 
Example~\ref{ex:running-example-1} and formally described in
Appendix~\ref{sec:running-example}.
Our algorithms output a near-optimal policy within a number of
steps that is polynomial in grid length $m$, whereas $\stm$ requires a number of
steps that is exponential in $m$.
The performance of our algorithms follows from Theorems~\ref{th:pac}
and \ref{th:additional-alg}, and from the fact that the parameters describing
$\rdp_m$ grow polynomially with $m$---please refer again to 
Appendix~\ref{sec:running-example} for a discussion of the parameters.
For the $\stm$ algorithm,
we argue that it cannot achieve arbitrary precision and confidence in polynomial
time, since the accuracy of its estimates depends on the probability of 
single histories of length $m$, which decreases exponentially with $m$.
Our argument to show a lower bound for the $\stm$ algorithm is based on the
fact that the probability of generating two identical histories up to the
$m$-th step in $N$ episodes can decrease exponentially with $m$---proof below.

\begin{proposition}
  \label{prop:comparison-1}
  Let $g = \max_i(p_i^0, p_i^1,1-p_i^0, 1-p_i^1)$.
  For every history $h \in S^m$ and every policy $\pi$, the
  probability of observing $h$ at least twice in $N$ episodes is at most
  $1/4 \cdot \euler^2 \cdot (\sqrt{2} \cdot g)^{2m} \cdot N^2$.
\end{proposition}

If $g < 1/\sqrt{2}$, and $N$ is polynomial in $m$, the former bound goes to zero
as $m$ increases, for $m$ sufficiently large.
Thus, for large values of $m$, there is a high probability that every history of
length $m$ occurs at most once.

We analyse the behaviour of the $\stm$ algorithm when every history of length
$m-2$ occurs at most once. These histories are important to determine the action
to take in column $m-1$ of the grid.
The algorithm compares pairs $\langle h, a_i \rangle$ of a history $h$ of length
$m-2$ and an action $a$ based on the empirical distribution on the next
observation; the only possible observations are of the form 
$\langle m-1, i, b \rangle$.
Thus,
the empirical distribution for $\langle h, a_i \rangle$ is either the
distribution $P_1$ that assigns probability one to $\langle m-1, i,
\mathit{enemy} \rangle$ or the distribution $P_2$ that assigns probability one
to $\langle m-1, i, \mathit{clear} \rangle$.
Let $\alpha = \min(p_{m-1}^0,p_{m-1}^1,1-p_{m-1}^0,1-p_{m-1}^1)$.
Regardless of wether $\langle h, a_i \rangle$ is assigned to $P_1$ or $P_2$,
there is probability at least $\alpha$ that the assigned distribution induces a
higher value-estimate for the worse action---e.g., if $a_i = a_0$ and enemy
$j$ is in cell $(0,j)$ with probability $p_{m-1}^0 > 1/2$, but we do not hit
enemy $j$.
In such a case, the error in the estimate is at least $\alpha$.
Therefore, the $\stm$ algorithm introduces an error of at least $\alpha$ with
probability at least $\alpha$.
For instance, if we take every $p_i^0 = 0.7$ and every $p_i^1 = 0.3$,
we have that $g = 0.7 < 1/\sqrt{2}$ as required by the proposition, and the
error of the $\stm$ algorithm is at least $0.3$ with probability at least $0.3$.

\begin{proof}[Proof of Proposition~\ref{prop:comparison-1}]
  First, note that the probability of a given observation 
  $\langle i, j, b \rangle$ upon performing action $a_k$ is at most $g$ if $j=k$
  and zero otherwise.
  We use a union bound over all histories.
  There are at most $4^m$ histories of length $m$ having non-zero probability.
  However, since the probability of actions $a_0$ and $a_1$ sums to one under
  every policy, and an observation has probability at most $g$ if the
  action index matches its second component (and probability zero otherwise), we
  can consider only one value for the second component of an observation and 
  bound the probability of the observation by $g$.
  Thus, the number of histories to consider in the union bound is $2^m$, with
  each history having probability at most $g^m$.
  Overall, the probability that any of the histories occurs in any two episodes,
  out of $N$ episodes, is at most:
  \begin{equation*}
    \sum_h \binom{N}{2} \cdot \prob(h)^2 = 
   \binom{N}{2} \cdot 2^m \cdot g^{2m} \leq \frac{\euler^2 \cdot N^2}{4}
   \cdot 2^m \cdot g^{2m} =  1/4 \cdot \euler^2 \cdot (\sqrt{2} \cdot g)^{2m} \cdot N^2,
 \end{equation*}
 where the binomial coefficient takes into account the possible orders of the
 episodes in which a history $h$ can occur. The inequality above makes use of
 the fact that $\binom{n}{k} \leq (\euler n/k)^k$.
\end{proof}

\fi

\end{document}